\documentclass[11pt,review]{elsarticle}

\usepackage{lineno,hyperref}
\modulolinenumbers[5]

\journal{European Journal of Operational Research}

\biboptions{authoryear}

\usepackage{mathtools} 
\usepackage{booktabs} 
\usepackage{tikz} 

\usepackage{soul}
\usepackage{url}
\usepackage{caption}
\usepackage{graphicx}
\usepackage{amsmath}
\usepackage{booktabs}
\usepackage{amssymb}
\usepackage[algo2e]{algorithm2e}
\usepackage{algorithm}
\usepackage{nicefrac}
\usepackage{amsfonts}
\usepackage{commath}
\usepackage{esvect}

\usepackage{longtable}
\usepackage{comment}

\newcommand{\Mmodel}{\mu}

\newcommand{\Mmodeldot}{\dot{\mu}}
\newcommand{\MmodeldotUnder}{\dot{\underline{\mu_i}}}

\newcommand{\Imatrixd}{\mathbf{I}_d}

\newcommand{\algotwo}{\textsc{CB-MNL}}
\newcommand{\algocompUCB}{\textsc{UCB-MNL}}
\newcommand{\algocompTS}{\textsc{TS-MNL}}

\newcommand{\MthetaProjectedt}{\theta_t^{\,\text{est}}}
\newcommand{\PredErr}{\Delta^{\text{pred}}}
\newcommand{\xsi}{x_{s,i}}
\newcommand{\xti}{x_{t,i}}
\newcommand{\sumInAssortS}{\sum_{i\in \mathcal{Q}_s}}
\newcommand{\sumInAssortSj}{\sum_{j\in \mathcal{Q}_s}}
\newcommand{\sumInAssortTj}{\sum_{j\in \mathcal{Q}_t}}

\newcommand{\sumInAssortT}{\sum_{i\in \mathcal{Q}_t}}
\newcommand{\designQt}{\mathbf{X}_{\mathcal{Q}_t}}
\newcommand{\designQtopt}{\mathbf{X}_{\mathcal{Q}^*_t}}
\newcommand{\designQs}{\mathbf{X}_{\mathcal{Q}_s}}

\DeclareMathOperator*{\argmin}{argmin}
\DeclareMathOperator*{\argmax}{argmax}

\usepackage{mathtools}
\usepackage{amsthm}

\newtheorem{thm}{Theorem}
\newtheorem{lemma}[thm]{Lemma}
\newtheorem{corol}[thm]{Corollary}

\newtheorem{remark}{Remark}
\newtheorem{assumption}{Assumption}

\usepackage{xcolor}
\newcommand{\PA}[1]{\textcolor{red}{PA: #1}}

\newtheorem{innercustomgeneric}{\customgenericname}
\providecommand{\customgenericname}{}
\newcommand{\newcustomtheorem}[2]{%
  \newenvironment{#1}[1]
  {%
   \renewcommand\customgenericname{#2}%
   \renewcommand\theinnercustomgeneric{##1}%
   \innercustomgeneric
  }
  {\endinnercustomgeneric}
}

\newcustomtheorem{customthm}{Theorem}
\newcustomtheorem{customlemma}{Lemma}
\newcustomtheorem{customcorollary}{Corollary}

\begin{document}

\begin{frontmatter}

\title{A Tractable Online Learning Algorithm for the Multinomial Logit Contextual Bandit}



\author[mymainaddress]{Priyank Agrawal\corref{mycorrespondingauthor}}
\cortext[mycorrespondingauthor]{Corresponding author}
\ead{pa2608@columbia.edu}

\author[mysecondaryaddress2]{Theja Tulabandhula}
\ead{tt@theja.org}

\author[mysecondaryaddress]{Vashist Avadhanula}
\ead{vas1089@gmail.com}

\address[mymainaddress]{500 W 120th St, New York, NY 10027}

\address[mysecondaryaddress2]{University Hall, 601 S Morgan St, Chicago, IL 60607}

\address[mysecondaryaddress]{1100 Enterprise Way, Sunnyvale, CA 94089}

\begin{abstract}
In this paper, we consider the contextual variant of the MNL-Bandit problem. More specifically, we consider a dynamic set optimization problem, where a decision-maker offers a subset (assortment) of products to a consumer and observes the response in every round. Consumers purchase products to maximize their utility. We assume that a set of attributes describe the products, and the mean utility of a product is linear in the values of these attributes. We model consumer choice behavior using the widely used Multinomial Logit (MNL) model and consider the decision maker’s problem of dynamically learning the model parameters while optimizing cumulative revenue over the selling horizon $T$. Though this problem has recently attracted considerable attention, many existing methods often involve solving an intractable non-convex optimization problem. Their theoretical performance guarantees depend on a problem-dependent parameter which could be prohibitively large. In particular, current algorithms for this problem have regret bounded by $O(\sqrt{\kappa d T})$, where $\kappa$ is a problem-dependent constant that may have an exponential dependency on the number of attributes, $d$. In this paper, we propose an optimistic algorithm and show that the regret is bounded by $O(\sqrt{dT} + \kappa)$, significantly improving the performance over existing methods. Further, we propose a convex relaxation of the optimization step, which allows for tractable decision-making while retaining the favorable regret guarantee. We also demonstrate that our algorithm has robust performance for varying $\kappa$ values through numerical experiments.
\end{abstract}


\end{frontmatter}



\section{Introduction}\label{sec: introduction main}
Assortment optimization problems arise in many industries, and prominent examples include retailing and online advertising (check \cite{ALFANDARI2021830,TIMONINAFARKAS20201058,WANG2020237} and see \cite{kok2007demand} for a detailed review). The problem faced by a decision-maker is that of selecting a subset (assortment) of items to offer from a universe of \emph{substitutable} items\footnote{If all consumers have identical preferences towards same characteristics of an item, then that item is termed as substitutable} such that the expected revenue is maximized. In many e-commerce applications, the data on consumer choices tends to be either limited or non-existent (similar to the \emph{cold start} problem in recommendation systems). Consumer preferences must be learned by experimenting with various assortments and observing consumer choices, but this experimentation with various assortments must be balanced to maximize cumulative revenue. Furthermore, in many settings, the retailer has to consider a very large number of products that are \emph{similar} (examples range from apparel to consumer electronics). The commonality in their features can be expressed with the aid of auxiliary variables which summarize product attributes. This enables a significant reduction in dimensionality but introduces additional challenges in designing policies that have to dynamically balance demand learning (exploration) while simultaneously maximizing cumulative revenues (exploitation). 

Motivated by these issues, we consider the dynamic assortment optimization problem. In every round, the retailer offers a subset (assortment) of products to a consumer and observes the consumer response. Consumers purchase (at most one product from each assortment) products that maximize their utility, and the retailer enjoys revenue from the successful purchase. We assume that the products are described by a set of attributes and the mean utility of a product is linear in the values of these attributes. We model consumer choice behavior using the widely used Multinomial Logit (MNL) model and consider the retailer's problem of dynamically learning the model parameters while optimizing cumulative revenues over the selling horizon $T$. Specifically, we have a universe of $N$ \emph{substitutable} items, and each item $i$ is associated with an attribute vector $x_i \in \mathbb{R}^d,$ which is known a priori. The mean utility for the consumer for the product $i$ is given by the inner product $\theta \cdot x_i,$ where $\theta \in \mathbb{R}^d$ is some fixed but initially unknown parameter vector. Each of the $d$ coordinates of $x_i$ for product $i$ represent a variety of characteristics such as cost, popularity, brand, etc. Given the \emph{substitutable} good assumption, the preference of all consumers towards these characteristics are identical and denoted by the same parameter\footnote{This assumption may appear quite restrictive at first. But, as described in the following paragraph and Section~\ref{sec: model setting}, the model is rich enough to model non-identical consumer behavior as well.} $\theta \in \mathbb{R}^d$. Further, any two products $i$ and $j$ could vary in terms of these characteristics and hence are associated with different vectors $x_i$ and $x_j \in \mathbb{R}^d$ respectively. Our goal is to offer assortments $\mathcal{Q}_1,\cdots, \mathcal{Q}_T$ at times $1,\cdots, T$ from a feasible collection of assortments such that the cumulative expected revenue of the retailer over the said horizon is maximized. In general, the feasible set of assortments can reflect the constraints of retailers and online platforms (such as cardinality, inventory availability and other related constraints).

For an intuitive understanding of the choice model, consider an example of an online furniture retailer that offers $N$ distinct products where the $i^{th}$ product has an attribute vector $x_i$ (in general, this attribute can vary over time, representing varying consumers' choices, and is more appropriately represented by $x_{t,i}$). 
Suppose consumers query for a specific product category, say \emph{tables}. In this example, the $\theta$ parameter will be a distinct vector corresponding to the product category: \emph{table}. As discussed before, th true $\theta_*$ that determines consumer choice behavior is unknown. With each interaction with the consumer, the online retailer is learning which of the $N$ products offers the most utility (captured by $\theta\cdot x_{i}$ for each product $i$) to the consumer by observing the past purchase decisions of the consumers. The online furniture retailer is constrained to offer at most $K$ of $N$ products in each interaction with the consumer. Such a constraint may be encountered in practical situations: limitation of the online consumer interface to display large number of products; consumer preferring to examine only a subset of products at a time etc. Out of $N$ furniture items, some particular \emph{table} $j$ could have high utility, $\theta\cdot x_j$, whereas, for a some product $k$ (say, a \emph{table} with unpopular color, bad design or inferior material etc.) $\theta\cdot x_k$ could be low. The consumer may purchase one or none of the $K$ presented products. Later in Section~\ref{sec: model setting}, we demonstrate that when the consumer's propensity to purchase a specific product is driven by its utility, the retailer's expected revenue at each round is given by a softmax function.

The rest of this section is organized as follows: We first describe the related literature and qualitative significance of the parameter $\kappa$. Then, we highlight our contributions and end the section by contrasting them with recent notable research works.

\subsection{Related literature} 

The MNL model is a widely used 
choice model for capturing consumer purchase behavior in assortment selection models (see \cite{FLORES20191052} and \cite{avadhanula2019mnl}). Recently, large-scale field experiments at Alibaba  \citep{feldman2018taking} have demonstrated the efficacy of the MNL model in boosting revenues. \cite{rusmevichientong2010dynamic} and \cite{saure2013optimal} were a couple of early works that studied  \emph{explore-then-commit} strategies for the dynamic assortment selection problem under the MNL model when there are no contexts/product features. The works of \cite{agrawal2019mnl} and \cite{agrawal2017thompson} revisited this problem and presented \emph{adaptive} online learning algorithms based on the Upper Confidence Bounds(UCB) and Thompson Sampling (TS) ideas. These approaches, unlike earlier ideas, did not require prior information about the problem parameters and had near-optimal regret bounds. Following these developments, the contextual variant of the problem has received considerable attention. \cite{cheung2017thompson} and \cite{oh2019thompson} propose TS-based approaches and establish Bayesian regret bounds on their performance\footnote{Our results give worst-case regret bound which is strictly stronger than Bayesian regret bound. Worst-case regret bounds directly imply Bayesian regret bounds with same order dependence.}. \cite{chen2018dynamic} present a UCB-based algorithm and establish min-max regret bounds. However, these contextual MNL algorithms and their performance bounds depend on a problem parameter $\kappa$ that can be prohibitively large, even for simple real-life examples. See Figure~\ref{fig:kappa_illustration} for an illustration and Section~\ref{sec: comparision with other works main} for a detailed discussion.\begin{figure}
\centering
\includegraphics[width=.75\textwidth]{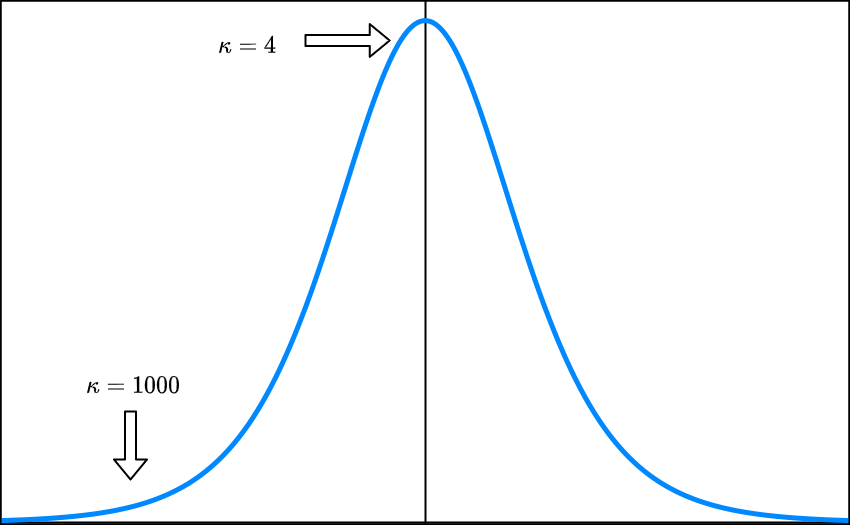}
\caption{Illustration of the impact of the $\kappa$ parameter (logistic case, multinomial logit case closely follows): A representative plot of the derivative of the reward function. The x-axis represents the linear function $x^\top\theta$ and the y-axis is proportional to $1/\kappa$. Parameter $\kappa$ is small only in the narrow region around $0$ and grows arbitrarily large depending on the problem instance (captured by $x^\top\theta$ values).}
\end{figure}\label{fig:kappa_illustration} We note that \cite{ou2018multinomial} also consider a similar problem of developing an online algorithm for the MNL model with linear utility parameters. Though they establish a regret bound that does not depend on the aforementioned parameter $\kappa$, they work with an \emph{inaccurate} version of the MNL model. More specifically, in the MNL model, the probability of a consumer preferring an item is proportional to the exponential of the utility parameter and is \emph{not} linear in the utility parameter as assumed in \cite{ou2018multinomial}.


The multi-armed bandit problem, which underlies these dynamic decision making settings, has been well studied in the literature (see \cite{XU2021622,GRANT2021575}). Our problem is closely related to the parametric bandit problem, where a common unknown parameter connects the rewards of each arm. In particular, for linear bandits, each arm $a \in A$ (consider $A$ to be the set of all arms) is associated with a $d$-dimensional vector $x_a \in \mathbb{R}^d$ known a priori. And the expected reward upon selecting arm $a \in A$ is given by the inner product $\theta \cdot x_a$, for some unknown parameter vector $\theta$ (see \cite{dani2008stochastic,rusmevichientong2010linearly,abbasi2011improved}). The key difference is that the rewards (i.e., the revenue of the retailer) corresponding to an assortment under the MNL cannot be modeled in the framework of linear payoffs.  Closer to our formulation is the literature on generalized linear bandits (see \cite{filippi2010parametric} and \cite{faury2020improved}), where the expected payoff upon selecting arm $a$ is given by $f(\theta\cdot x_a)$, where $f$ is a real-valued, non-linear function. However, unlike our setting, where an \emph{arm} could be a collection of $K$  products (thus involving $K$ $d$-dimensional vectors), $f(.)$ is a single variable function in these prior works. 

\subsection{On the parameter $\kappa$ }\label{sec: comparision with other works main}

As discussed earlier, the retailer's revenue (reward function) is the softmax function. Intuitively, the curvature of the reward function influences how easy (or difficult) it is to learn the true choice parameter $\theta_*$. In Section~\ref{sec: preliminaries main}, we explicitly define $\kappa$ as inversely proportional to the lower bound on of the \emph{derivative} of the reward function in the entire decision region. Existence of a global lower bound on the curvature of the reward function is a necessary assumption for the maximum likelihood estimation of $\theta_*$.

In previous works on generalized linear bandits and variants~\citep{filippi2010parametric,li2017provably,oh2019thompson}, the quantity $\kappa$ features in regret guarantees as a multiplicative factor of the primary term (i.e., as $\tilde{\mathrm{O}}(\kappa\sqrt{T})$), and this is because they ignore the local effect of the curvature, and use global properties (via $\kappa$) leading to loose worst-case bounds. For a cleaner exposition of this issue, lets take $K=1$, i.e., the rewards are given by a sigmoid function of $\theta\cdot x$. The derivative of sigmoid is ``bell''-shaped (see Figure~\ref{fig:kappa_illustration}). When $\theta\cdot x$ is very high (i.e., the assortment contains products with high utilities) or when $\theta\cdot x$ is very low (i.e., the assortment contains products with low utility), the value of $\kappa$ will be large. From Assumption~\ref{assm: kappa assumption}, for $K=1$, $\kappa$ is equivalent to $\max \frac{1}{a(1-a)}$, for some $a\in(0,1)$. Thus, when $a$ is close to $1$ or $0$, the value of $\kappa$ will be large. The exponential dependence for $K=1$ case follows when we replace $a$ with a sigmoid function. In the context of our problem, this translates to an exponential dependence of the per-round regret on the magnitude of utilities (i.e., $\theta\cdot x$). 




\subsection{Contributions} 
In this paper, we build on recent developments for generalized linear bandits (\cite{faury2020improved}) to propose a new optimistic algorithm, \algotwo{} for the problem of contextual multinomial logit bandits. \algotwo{} follows the standard template of optimistic parameter search strategies (also known as \emph{optimism in the face of uncertainty} approaches) ~\citep{abbasi2011improved,abeille2020instance}. We use Bernstein-style concentration for self-normalized martingales, which were previously proposed in the context of scalar logistic bandits in~\cite{faury2020improved}, to define our confidence set over the true parameter, taking into account the effects of the local curvature of the reward function. We show that the performance of \algotwo{} (as measured by regret) is bounded as $\tilde{\mathrm{O}}\del{d\sqrt{T} + \kappa}$, significantly improving the theoretical performance over existing algorithms where $\kappa$ appears as a multiplicative factor in the leading term. We also leverage a self-concordance~\citep{bach2010self} like relation for the multinomial logit reward function~\citep{zhang2015disco}, which helps us limit the effect of $\kappa$ on the final regret upper bound to only the higher-order terms. Finally, we propose a different convex confidence set for the optimization problem in the decision set of~\algotwo{}, which reduces the optimization problem to a constrained convex problem. 

In summary, our work establishes strong worst-case regret guarantees by carefully accounting for local gradient information and using \emph{second-order} function approximation for the estimation error. 

\subsection{Comparison with notable prior works}\label{sec: comparison with specific literature}


\noindent \textbf{Comparison with~\cite{filippi2010parametric}} Our setting is different from the standard generalized linear bandit of~\cite{filippi2010parametric}. In our setting, the reward due to an action (assortment) can be dependent on up to $K$ variables ($\theta_*\cdot x_{t,i},\,i\in \mathcal{Q}_t$) instead of a single variable. Further, we focus on removing the multiplicative dependence on $\kappa$ from the regret bounds. This leads to a more involved technical treatment in our work.

\noindent \textbf{Comparison with~\cite{oh2019thompson}} The Thompson Sampling based approach is inherently different from our \emph{Optimism in the face of uncertainty (OFU)} style Algorithm~\algotwo{}. However, the main result in~\cite{oh2019thompson} also relies on a confidence set based analysis along the lines of~\cite{filippi2010parametric} but has a multiplicative $\kappa$ factor in the bound.

\noindent \textbf{Comparison with~\cite{faury2020improved}} \cite{faury2020improved} use a bonus term for optimization in each round, and their algorithm performs non-trivial projections on the admissible log-odds. While we do reuse the Bernstein-style concentration inequality as proposed by them, their results do not seem to extend directly to the MNL setting without requiring significantly more work. Further, our algorithm~\algotwo{} performs an optimistic parameter search for making decisions instead of using a bonus term, which allow for a cleaner and shorter analysis.

\noindent \textbf{Comparison with~\cite{oh2021multinomial}} While the authors in~\cite{oh2021multinomial} provide sharper bounds by a factor of $\tilde{\mathrm{O}}(\sqrt{d})$, they still retain the $\kappa$ multiplicative factor in their regret bounds. Their focus is on improving the dependence on the dimension parameter $d$ for the dynamic assortment optimization problem. 

\noindent \textbf{Comparison with~\cite{abeille2020instance}} ~\cite{abeille2020instance} recently proposed the idea of convex relaxation of the confidence set for the more straightforward logistic bandit setting. Our work can be viewed as an extension of their construction to the MNL setting.

\noindent \textbf{Comparison with~\cite{amani2021ucb}} While the authors in~\cite{amani2021ucb} also extend the algorithms of~\cite{faury2020improved} to a multinomial problem, their setting is materially different from ours. They model various click-types for the same advertisement (action) via the multinomial distribution. further, they consider actions played at each round to be non-combinatorial, i.e., a single action as opposed to a bundle of actions, which differs from the assortment optimization setting in this work. Therefore, their approach and technical analysis are different from ours.


\section{Preliminaries}\label{sec: preliminaries main}
\subsection{Notations}\label{sec: notation main}
For a vector $x\,\in\,\mathbb{R}^d$, $x^\top$ denotes the transpose. Given a positive definite matrix $\mathbf{M}\,\in\,\mathbb{R}^{d \times d}$, the induced norm is given by $||x||_{\mathbf{M}} = \sqrt{x\mathbf{M}x}$. For two symmetric matrices $\mathbf{M_1}$ and $\mathbf{M_2}$, $\mathbf{M_1} \succeq \mathbf{M_2}$ means that $\mathbf{M_1}-\mathbf{M_2}$ is positive semi-definite. For any positive integer $n$, $[n]\coloneqq \{1,2,3,\cdots,n\}$. $\Imatrixd$ denotes an identity matrix of dimension $d\times d$. The platform (i.e. the learner) is referred using the pronouns she/her/hers.

\subsection{Model setting}\label{sec: model setting}

\paragraph{Rewards Model:} At every round $t$, the platform (learner) is presented with set $\mathcal{N}$ of distinct items, indexed by $i\,\in\,[N]$ and their attribute vectors (contexts): $\{x_{t,i}\}_{i=1}^{N}$ such that $\forall \, i\,\in [N],\,x_{t,i}\,\in\,\mathbb{R}^d$, where $N=|\mathcal{N}|$ is the cardinality of set $\mathcal{N}$. The platform then selects an assortment $\mathcal{Q}_t \subset \mathcal{N}$ and the interacting consumer (environment) offers the reward $r_t$ to the platform. The assortments have a cardinality of at most $K$, i.e. $|\mathcal{Q}_t|\leq K$. The platform's decision is based on the entire history of interaction. The history is represented by the filtration set $\mathcal{F}_t\coloneqq\{\mathcal{F}_0,\sigma(\{\{x_{s,i}\}_{i= 1}^N,\mathcal{Q}_s\}_{s=1}^{t-1})\}$,\footnote{$\sigma(\{\cdot\})$ denotes the $\sigma$-algebra set over the sequence $\{\cdot\}$.} where $\mathcal{F}_0$ is any prior information available to the platform. The interaction lasts for $t=1,2,\cdots,T$ rounds. Conditioned on $\mathcal{F}_t$, the reward $r_t$ is a binary vector such that $r_t\,\in\,\{0,1\}^{N}$ and the vector $\{r_{t,i}\}_{i\in\mathcal{Q}_t}$ follows a multinomial distribution. We have $r_{t,i}=0,\forall \,i\,\notin\,\mathcal{Q}_t$.  Specifically, the probability that $r_{t,i} = 1,\forall\,i\,\in\,\mathcal{Q}_t$ is given by the softmax function:
\begin{align}\label{eq: rewards eq}
\mathbb{P}(r_{t,i)=1 | \mathcal{Q}_t,\mathcal{F}_t } = \mu_{i}(\mathcal{Q}_t,\theta_*) \coloneqq \frac{\exp(x_{t,i}^\top\theta_*)}{1+\sum_{j\in\mathcal{Q}_t}\exp(x_{t,j}^\top\theta_*)},
\end{align}
where $\theta_*$ is an unknown time-invariant parameter. The numeral $1$ in the denominator accounts for the case when the consumer purchases none of the items in the assortment. By definition, $\sum_{i\in\mathcal{Q}_t} r_{t,i} \leq 1$, i.e., $r_t$ is multinomial with a single trial. Also, the expected revenue due to the assortment\footnote{Each item $i$ is also associated with a price (or revenue) parameter, $p_{t,i}$ for round $t$. We assume $p_{t,i}=1$ for all items and rounds for an uncluttered exposition of results. If $p_{t,i}$ is not $1$, then it features as a fixed factor in the definition of $\mu_i(\cdot)$ and the analysis exactly follows as that presented here $p_{t,i}=1$ for all rounds and items.} $\mathcal{Q}_t$ is given by:
\begin{equation}\label{eq: total revenue}
\mu(\mathcal{Q}_t,\theta_*)\coloneqq\sum_{i\in\mathcal{Q}_t}\mu_{i}(\mathcal{Q}_t,\theta_*).
\end{equation}
Also, $\{x_{t,i}\}$ may vary adversarially in each round in our model, unlike in~\cite{li2017provably}, where the attribute vectors are assumed to be drawn from an unknown i.i.d. distribution. When $K=1$, the above model reduces to the case of the logistic bandit.  

\paragraph{Choice Modeling Perspective:} Eq~\ref{eq: rewards eq} can be considered from a discrete choice modeling viewpoint, where the platform presents an assortment of items to a user, and the user selects at most one item from this assortment. In this interpretation, the probability of choosing an item $i$ is given by $\Mmodel_i(\mathcal{Q}_t,\theta_*)$. Likewise, the probability of the user not selecting any item is given by: $\nicefrac{1}{(1+\sum_{j\in\mathcal{Q}_t}\exp(x_{t,j}^\top\theta_*))}$. The platform is motivated to offer such an assortment that the user's propensity to make a successful selection is high.

\paragraph{Regret:} The platform does not know the value of $\theta_*$. Our learning algorithm~\algotwo{} (see Algorithm~\ref{alg: ML UCB}) sequentially makes the assortment selection decisions, $\mathcal{Q}_1,\mathcal{Q}_2,\cdots, \mathcal{Q}_T$ so that the cumulative expected revenue $\sum_{t=1}^T\mu(\mathcal{Q}_t,\theta_*)$ is high. Its performance is quantified by \textit{pseudo-regret}, which is the gap between the expected revenue generated by the algorithm and that of the optimal assortments in hindsight. The learning goal is to minimize the cumulative pseudo-regret up to time $T$, defined as:
\begin{equation}\label{eq: regret expression}
    \mathbf{R}_T \coloneqq \sum_{t=1}^T [\Mmodel(\mathcal{Q}_t^*,\theta_*)-\Mmodel(\mathcal{Q}_t,\theta_*)],
\end{equation}
where $\mathcal{Q}_t^*$ is the offline optimal assortment at round $t$ under full information of $\theta_*$, defined as:
$
\mathcal{Q}_t^* \coloneqq \argmax_{\mathcal{Q}\subset \mathcal{N}}\Mmodel(\mathcal{Q},\theta_*).
$

As in the case of contextual linear bandits~\cite{abbasi2011improved,chu2011contextual}, the emphasis here is to make good sequential decisions while tracking the true parameter $\theta_*$ with a close estimate $\hat{\theta}_t$ (see Section~\ref{sec: MLE main}). Our algorithm (like others) does not necessarily \emph{improve} the estimate at each round. However, it ensures that $\theta_*$ is always within a confidence interval of the estimate of $\theta_*$ (with high probability) and the future analysis demonstrates that the aggregate prediction error over all $T$ rounds is bounded.

Our model is fairly general, as the contextual information $x_{t,i}$ may be used to model combined information of the item $i$ in the set $\mathcal{N}$ and the user at round $t$. Suppose the user at round $t$ is represented by a vector $v_t$ and the item $i$ has attribute vector as $w_{t,i}$, then $x_{t,i} = \text{vec}(v_tw_{t,i}^\top)$ (vectorized outer product of $v_t$ and $w_{t,i}$). We assume that the platform knows the interaction horizon $T$.\\ \textbf{Additional notations:} $\designQt$ denotes a design matrix whose columns are the attribute vectors ($x_{t,i}$) of the items in the assortment $\mathcal{Q}_t$. Also, we now denote $\mu(\mathcal{Q}_t,\theta_*)$ as $\mu(\designQt^\top\theta_*)$ to signify that $\mu(\mathcal{Q}_t,\theta_*) : \mathbb{R}^{|\mathcal{Q}_t|}\to \mathbb{R}$.

\subsection{Assumptions}\label{sec: assumptions main}
Following~\cite{filippi2010parametric,li2017provably,oh2019thompson,faury2020improved}, we introduce the following assumptions on the problem structure. 
\begin{assumption}[Bounded parameters]\label{assm: bounded parameters assumption}
$\theta_*\,\in\,\Theta$, where $\Theta$ is a compact subset of $\mathbb{R}^d$. $S\coloneqq \max_{\theta\in\Theta}||\theta||_2$ is known to the learner. Further, $||x_{t,i}||_2\leq 1$ for all values of $t$ and $i$.
\end{assumption}
This assumption simplifies analysis and removes scaling constants from the equations.

\begin{assumption}\label{assm: kappa assumption}
There exists $\kappa > 0$ such that for every item $i\,\in\,\mathcal{Q}_t$ and for any $\mathcal{Q}_t\subset \mathcal{N}$ and all rounds $t$:
$$
\inf_{\mathcal{Q}_t\subset \mathcal{N},\theta \in \mathbb{R}^d} \Mmodel_i(\designQt^\top\theta)(1-\Mmodel_i(\designQt^\top\theta))\geq \frac{1}{\kappa}. 
$$
\end{assumption}
Note that $\Mmodel_i(\designQt^\top\theta)(1-\Mmodel_i(\designQt^\top\theta))$ denotes the derivative of the softmax function along the $i_{th}$ direction. This assumption is necessary from the likelihood theory~\cite{lehmann2006theory} as it ensures that the fisher matrix for $\theta_*$ estimation is invertible for all possible input instances. We refer to~\cite{oh2019thompson} for a detailed discussion in this regard. We denote $L$ and $M$ as the upper bounds on the first and second derivatives of the softmax function along any component, respectively. We have $L,M\leq 1$~\citep{gao2017properties} for all problem instances.

\subsection{Maximum likelihood estimate}\label{sec: MLE main}
\algotwo, described in Algorithm~\ref{alg: ML UCB}, uses a regularized maximum likelihood estimator to compute an estimate $\hat{\theta}_t$ of $\theta_*$. Since $\{r_{t,i}\}_{i\in\mathcal{Q}_t}$ follows a multinomial distribution, the regularized log-likelihood (negative cross entropy loss) function, till the $(t-1)_{th}$ round, under parameter $\theta$ could be written as:
\begin{align}\label{eq: log liklehood main}
\mathcal{L}_t^{\lambda_t}(\theta) = \sum_{s=1}^{t-1}\sum_{i\in \mathcal{Q}_s} r_{s,i}\log(\mu_i(\designQs^\top \theta)) - \frac{\lambda_t}{2}||\theta||_2^2,
\end{align}
$\mathcal{L}_t^{\lambda_t}(\theta)$ is concave in $\theta$ for $\lambda_t >0$, and the maximum likelihood estimator is given by calculating the critical point of $\mathcal{L}_t^{\lambda_t}(\theta)$. Setting $\nabla_{\theta}\mathcal{L}_t^{\lambda_t}(\theta) = 0$, we get $\hat{\theta}_t$ as the solution of:
\begin{equation}\label{eq: MLE estimate}
    \sum_{s=1}^{t-1}\sumInAssortS [\mu_i(\designQs^\top \hat{\theta}_t) - r_{t,i}]\xsi + \lambda_t\hat{\theta}_t = 0. 
\end{equation}
For future analysis we also define
\begin{align}
g_t(\theta) \coloneqq \sum_{s=1}^{t-1}\sumInAssortS \mu_i(\designQs^\top \theta) \xsi + \lambda_t\theta,\quad 
g_t(\hat{\theta}_t) \coloneqq \sum_{s=1}^{t-1}\sumInAssortS r_{s,i}\xsi.\label{eq: gt main}
\end{align}
At the start of the interaction, when no contexts have been observed, $\hat{\theta}_t$ is well-defined by Eq~(\ref{eq: MLE estimate}) when $\lambda_t>0$. Therefore, the regularization parameter $\lambda_t$ makes \algotwo{} burn-in period free, in contrast to some previous works, e.g.~\cite{filippi2010parametric}.

\subsection{Confidence sets}\label{sec: confidence set main}
Algorithm~\ref{alg: ML UCB} follows the template of \emph{in the face of uncertainty (OFU)} strategies~\citep{auer2002finite,filippi2010parametric,faury2020improved}. Technical analysis of \emph{OFU} algorithms relies on two key factors: the design of the confidence set and the ease of choosing an action using the confidence set.

In Section~\ref{sec: main results}, we derive $E_t(\delta)$ (defined below) as the confidence set on $\theta_*$ such that $\theta_*\in C_t(\delta),\,\forall t$ with probability at least $1-\delta$ (randomness is over user choices). $E_t(\delta)$ used for making decisions at each round (see Eq~(\ref{eq: algo decision})) by \algotwo{} in Algorithm~\ref{alg: ML UCB}:
\begin{equation}\label{eq: convex relaxation set prev}
 E_t(\delta) \coloneqq \{ \theta\in\Theta, \, \mathcal{L}^{\lambda_t}_t(\theta)- \mathcal{L}^{\lambda_t}_t(\hat{\theta}_t) \leq \beta^2_t(\delta) \},
\end{equation}
where $\beta_t(\delta) \coloneqq \gamma_t(\delta)+\frac{\gamma_t^2(\delta)}{\lambda_t}$,
and
\begin{align}\label{eq: gamma value main}
\gamma_t(\delta) \coloneqq& \frac{\sqrt{\lambda_t}}{2} +  \frac{2}{\sqrt{\lambda_t}}\log(\frac{(\lambda_t +L Kt/d)^{d/2}\lambda_t^{-d/2}}{\delta}) + \frac{2d}{\sqrt{\lambda_t}}\log(2).
\end{align}

A confidence set similar to $E_t(\delta)$ in Eq~(\ref{eq: convex relaxation set prev}) was recently proposed in~\cite{abeille2020instance} for the simpler logisitic bandit setting. Here, we extend its construction to the MNL setting. The set $E_t(\delta)$ is convex since the log-loss function is convex. This makes the decision step in Eq~(\ref{eq: algo decision}) a constraint convex optimization problem. However, it is difficult to prove bounds directly with $E_t(\delta)$. Therefore we leverage a result in~\cite{faury2020improved}, where the authors proposed a new Bernstein-like tail inequality for self-normalized vectorial martingales (see Appendix A.1), to derive another confidence set on $\theta_*$:
\begin{equation}\label{eq: confidence set def main}
    C_t(\delta) \coloneqq \{ \theta\in\Theta, \, ||g_t(\theta) - g_t(\hat{\theta}_t)||_{\mathbf{H}_t^{-1}(\theta)} \leq \gamma_t(\delta) \}.
\end{equation}
where 
\begin{equation}\label{eq: Hm norm main}
\mathbf{H}_t(\theta_1) \coloneqq \sum_{s=1}^{t-1}\sumInAssortS\Mmodeldot_i(\designQs^\top\theta_1)\xsi \xsi^\top+\lambda_t\Imatrixd.
\end{equation}
$\Mmodeldot_i(\cdot)$ is the partial derivative of $\Mmodel_i$ in the direction of the $i_{th}$ component of the assortment and $\gamma_t(\delta)$ is defined in Eq~(\ref{eq: gamma value main}). The value of $\gamma_t(\delta)$ is an outcome of the concentration result of~\cite{faury2020improved}. As a consequence of this concentration, we have $\theta_*\,\in\,C_t(\delta)$ with probability at least $1-\delta$ (randomness is over user choices). The Bernstein-like concentration inequality used here is similar to Theorem 1 of~\cite{abbasi2011improved} with the difference that we take into account local variance information (hence local curvature information of the reward function) in defining $\mathbf{H}_t$. The above discussion is formalized in Appendix A.1. 

The set $C_t(\delta)$ is non-convex, which follows from the non-linearity of $\mathbf{H}^{-1}_t(\theta)$. We use $C_t(\delta)$ directly to prove regret guarantees. In Section~\ref{sec: convex relaxation main}, we mention how the a convex set $E_t(\delta)$ is related to $C_t(\delta)$ and share many useful properties of $C_t(\delta)$. Till then, to maintain ease of technical flow and to compare it with the previous work~\cite{faury2020improved}, we assume that the algorithm uses $C_t(\delta)$ as the confidence set. We highlight that for the confidence sets, $C_t(\delta)$ and $E_t(\delta)$, Algorithm~\algotwo{} is identical except for the calculation in Eq~(\ref{eq: algo decision}). For later sections we also define the following norm inducing design matrix based on all the contexts observed till time $t-1$:
\begin{equation}\label{eq: vm definition main}
\mathbf{V}_t \coloneqq \sum_{s=1}^{t-1}\sumInAssortS\xsi \xsi^\top+\lambda_t\Imatrixd.
\end{equation}

\section{Algorithm}\label{sec: algorithm main}
At each round $t$, the attribute parameters (contexts) $\{x_{t,1},x_{t,2},\cdots,x_{t,N}\}$ are made available to the algorithm (online platform) \algotwo. The algorithm calculates an estimate of the true parameter $\theta_*$ according to Eq~(\ref{eq: MLE estimate}). The algorithm keeps track of the confidence set $C_t(\delta)$ ($E_t(\delta)$) as defined in Eq~(\ref{eq: confidence set def main}) (Eq~(\ref{eq: convex relaxation set prev}). Let the set $\mathbf{\mathcal{A}}$ contain all feasible assortments of $\mathcal{N}$ with cardinality up to $K$. The algorithm makes the following decision:
\begin{equation}\label{eq: algo decision}
(\mathcal{Q}_t,\theta_t) = \argmax_{A_t \in \mathcal{A}, \theta \in C_t(\delta)} \Mmodel(\mathbf{X}_{A_t}^\top\theta).
\end{equation}
In each round $t$, the reward of the online platform is denoted by the vector $r_t$. Also, the \emph{prediction error} of $\theta$ at $\designQt$, defined as:
\begin{equation}\label{eq: prediction error main}
\PredErr(\designQt,\theta)\coloneqq|\Mmodel(\designQt^\top\theta_*)-\Mmodel(\designQt^\top\theta)|.
\end{equation}
$\PredErr\del{\designQt,\theta}$ represents the difference in perceived rewards due to the inaccuracy in the estimation of the parameter $\theta_*$.
\begin{remark}[Optimistic parameter search]
\algotwo{} enforces optimism via an optimistic parameter search (e.g. in~\cite{abbasi2011improved}), which is in contrast to the use of an exploration bonus as seen in~\cite{faury2020improved,filippi2010parametric}. Optimistic parameter search provides a cleaner description of the learning strategy. In non-linear reward models, both approaches may not follow similar trajectory but may have overlapping analysis styles (see~\cite{filippi2010parametric} for a short discussion).
\end{remark}

\begin{remark}[Tractable decision-making]
In Section~\ref{sec: convex relaxation main}, we show that the decision problem of Eq~(\ref{eq: algo decision}) can be relaxed to an convex optimization problem by using a convex set $E_t(\delta)$, instead of $C_t(\delta)$, while keeping the regret performance of Algorithm~\ref{alg: ML UCB} intact up to constant factors.
\end{remark}

\begin{algorithm}[ht]
	\SetAlgoLined
	\textbf{Input:} regularization parameters: $\lambda_t,\forall\, t\, \in \,[T]$, $N$ distinct items: $\mathcal{N}$, $K$\\
	\For{$t\,\geq\,1$ }{
	    {\bf Given:} Set $\{x_{t,1},x_{t,2},\cdots,x_{t,N}\}$ of $d$-dimensional parameters.\\
	    Estimate $\hat{\theta}_t$ according to Eq~(\ref{eq: MLE estimate}).\\
	    Construct $C_t(\delta)$ as defined in Eq~(\ref{eq: confidence set def main}).\\
        Construct the set $\mathbf{\mathcal{A}}$ of all feasible assortments of $\mathcal{N}$ with cardinality upto $K$.\\
        Play $(\mathcal{Q}_t,\theta_t) = \argmax_{A_t \in \mathcal{A}, \theta \in C_t(\delta)}\Mmodel(\mathbf{X}_{A_t}^\top\theta)$.\\
        Observe rewards $\mathbf{r}_t$
	}
	\caption{\algotwo}\label{alg: ML UCB}
\end{algorithm}

\section{Main results}\label{sec: main results}
We present a regret upper bound for the \algotwo{} algorithm in Theorem~\ref{thm: regret upper bounds main}.
\begin{thm}\label{thm: regret upper bounds main}
With probability at least $1-\delta$ over the randomness of user choices:
\begin{align*}
    \mathbf{R}_T \leq& C_1\gamma_T(\delta)\sqrt{2d\log(1+\frac{L KT}{d\lambda_T})T}+ C_2\kappa \gamma_T(\delta)^2d\log(1+\frac{KT}{d\lambda_T}),
\end{align*}
where the constants are given as $C_1 = (4+8S)$, $C_2= 4(4+8S)^{\nicefrac{3}{2}}M$, and $\gamma_T(\delta)$ is given by Eq~(\ref{eq: gamma value main}).
\end{thm}
The formal proof is deferred to the technical Appendix, in this section we discuss the key technical ideas leading to this result. The order dependence on the model parameters is made explicit by the following corollary.
\begin{corol}\label{corol: simplified regret result}
Setting the regularization parameter $\lambda_T = \mathrm{O}(d \log(K T))$, where $K$ is the maximum cardinality of the assortments to be selected, makes $\gamma_T(\delta) = \mathrm{O}(d^{\nicefrac{1}{2}}\log^{\nicefrac{1}{2}}(KT))$. The regret upper bound is given by $\mathbf{R}_T = \mathrm{O}(d \sqrt{T}\log(KT)+\kappa d^2 \log^2(KT)) $.
\end{corol}
Recall the expression for cumulative regret
\begin{align*}
    \mathbf{R}_T &= \sum_{t=1}^T [\Mmodel(\designQtopt^\top\theta_*)-\Mmodel(\designQt^\top\theta_*)]\\
     &= \sum_{t=1}^T \underbrace{[\Mmodel(\designQtopt^\top\theta_*)-\Mmodel(\designQt^\top\theta_t)]}_{\text{pessimism}} + \sum_{t=1}^T \underbrace{[\Mmodel(\designQt^\top\theta_t)-\Mmodel(\designQt^\top\theta_*)]}_{\text{prediction error}},
\end{align*}
where \emph{pessimism} is the additive inverse of the optimism (difference between the payoffs under true parameters and those estimated by \algotwo{}). Due to optimistic decision-making and the fact that $\theta_*\in C_t(\delta)$ (see Eq~(\ref{eq: algo decision})), \emph{pessimism} is non-positive, for all rounds. Thus, the regret is upper bounded by the sum of the prediction error for $T$ rounds. In Section~\ref{sec: bounds on prediction error main} we derive an the expression for prediction error upper bound for a single round $t$. We also contrast with the previous works~\cite{filippi2010parametric,li2017provably,oh2021multinomial} and point out specific technical differences which allow us to use Bernstein-like tail concentration inequality and therefore, achieve stronger regret guarantees. In Section~\ref{sec: regret calculation main}, we describe the additional steps leading to the statement of Theorem~\ref{thm: regret upper bounds main}. The style of the arguments is simpler and shorter than that in~\cite{faury2020improved}. Finally, in Section~\ref{sec: convex relaxation main}, we discuss the relationship between two confidence sets $C_t(\delta)$ and $E_t(\delta)$ and show that even using $E_t(\delta)$ in place of $C_t(\delta)$, we get the regret upper bounds with same parameter dependence as in Corollary~\ref{corol: simplified regret result}. 
Lemma~\ref{lem: prediction error single round softmax main} gives the expression for an upper bound on the prediction error.
\subsection{Bounds on prediction error}\label{sec: bounds on prediction error main}
\begin{lemma}\label{lem: prediction error single round softmax main}
For $\theta_t\in C_t(\delta)$ (see Eq~(\ref{eq: algo decision})) with probability at least $1-\delta$:
\begin{align}\label{eq: prediction error main 55}
\PredErr(\designQt,\theta_t)
\leq & (2+4S)\gamma_t(\delta)\sumInAssortT\Mmodeldot_i(\designQt^\top\theta_*)||x_{t,i}||_{\mathbf{H}_t^{-1}(\theta_*)}\nonumber\\ &+ 4\kappa(1+2S)^2M \gamma_t(\delta)^2\sumInAssortS||x_{t,i}||^2_{\mathbf{V}_t^{-1}},
\end{align}
where $\mathbf{V}_t^{-1}$ is given by Eq~(\ref{eq: vm definition main}).
\end{lemma}
The detailed proof is provided in~\ref{sec: bounds on prediction error apendix}. Here we develop the main ideas leading to this result and develop an analytical flow which will be re-used while working with convex confidence set $E_t(\delta)$ in Section~\ref{sec: convex relaxation main}. In the previous works~\cite{filippi2010parametric,li2017provably,oh2021multinomial}, global upper and lower bounds of the derivative of the link function (here softmax) are employed early in the analysis, leading to loss of local information carried by the MLE estimate $\theta_t$. In those previous works the first step was to upper bound the prediction error by the Lipschitz constant (which is a global property) of the softmax (or sigmoid for the logistic bandit case) function, as:
{\small \begin{equation}\label{eq: prediction error main 2}
    |\Mmodel(\designQt^\top\theta_*)-\Mmodel(\designQt^\top\theta_t)|\leq L|\designQt^\top(\theta_*-\theta_t)|.
\end{equation}}
For building intuition, assume that $\designQt^\top\theta_*$ lies on ``flatter'' region of $\Mmodel(\cdot)$, then Eq~(\ref{eq: prediction error main 2}) is a loose upper bound. 

Next we show how using a global lower bound in form of $\kappa$ (see Assumption~\ref{assm: kappa assumption}) early in the analysis in the works~\cite{filippi2010parametric,li2017provably,oh2021multinomial} lead to loose prediction error upper bound. For this we first introduce a new notation:
\begin{align}\label{eq: notation alpha main}
    &\alpha_i(\designQt,\theta_t,\theta_*)\xti^\top(\theta_*-\theta_t)\coloneqq \mu_{i}(\designQt^\top\theta_*)-\mu_{i}(\designQt^\top\theta_t). 
\end{align}
We also define $\mathbf{G}_t(\theta_t,\theta_*) \coloneqq \sum_{s=1}^{t-1}\sumInAssortS\alpha_i(\designQs,\theta_t,\theta_*)\xsi \xsi^\top+\lambda\Imatrixd.$ From Eq~(\ref{eq: gt main}), we obtain (see~\ref{sec: Local Information Preserving Norm appendix} for details of this derivation):
{\small \begin{align}\label{eq: proof outline 1}
    g(\theta_*) -g(\theta_t) 
    =\mathbf{G}_t(\theta_t,\theta_*)(\theta_*-\theta_t).
\end{align}}
From Assumption~\ref{assm: kappa assumption}, $\mathbf{G}_t(\theta_t,\theta_*)$ is a positive definite matrix for $\lambda > 0$ and therefore can be used to define a norm. Using Cauchy-Schwarz inequality with Eq~(\ref{eq: proof outline 1}) simplifies the prediction error as:
{\small \begin{align}\label{eq: prediction error eq main 1 new}
&\PredErr(\designQt,\theta_t) \leq\big|\sumInAssortT\alpha_i(\designQt,\theta_t,\theta_*)||x_{t,i}||_{\mathbf{G}_t^{-1}(\theta_t,\theta_*)}||\theta_*-\theta_t||_{\mathbf{G}_t(\theta_t,\theta_*)}\big|
\end{align}}
The previous literature~\cite{filippi2010parametric,oh2021multinomial} has utilized $\mathbf{G}_t^{-1}(\theta_t,\theta_*) \succeq \kappa^{-1}\mathbf{V}_t$ and upper bounded $\alpha_i(\designQt,\theta_t,\theta_*)$ by Lipschitz constant (directly at this stage), thereby incurring loose regret bounds. Instead, here we work with the norm induced by $\mathbf{H}_t(\theta_*)$ and retain the location information in $\alpha_i(\designQt,\theta_t,\theta_*)$. 
\begin{align}\label{eq: prediction error eq main 1}
&\PredErr(\designQt,\theta_t) \leq\big|\sumInAssortT\alpha_i(\designQt,\theta_t,\theta_*)||x_{t,i}||_{\mathbf{H}_t^{-1}(\theta_*)}||\theta_*-\theta_t||_{\mathbf{H}_t(\theta_*)}\big|
\end{align}
It is not straight-forward to bound $||\theta_*-\theta_t||_{\mathbf{H}_t(\theta_*)}$, we extend the self-concordance style relations from~\cite{faury2020improved} for the multinomial logit function which allow us to relate $\mathbf{G}_t^{-1}(\theta_t,\theta_*)$ and $\mathbf{H}_t^{-1}(\theta_t)$ (or $\mathbf{H}_t^{-1}(\theta_*)$) to develop a bound on $||\theta_*-\theta||_{\mathbf{H}_t(\theta_*)}$.
\begin{lemma}\label{lem: relationship between Gt and Ht main}
For all $\theta_1,\theta_2\,\in\,\Theta$, the following inequalities hold:
\begin{align*}
\mathbf{G}_t(\theta_1,\theta_2) \succeq (1+2S)^{-1}\mathbf{H}_t(\theta_1),\quad\mathbf{G}_t(\theta_1,\theta_2) \succeq (1+2S)^{-1}\mathbf{H}_t(\theta_2)
\end{align*}
\end{lemma}
\begin{lemma}\label{lem: projection with Gt and Ht relation main}
For $\theta_t \in C_t(\delta) $, we have the following relation with probability at least $1-\delta$: $||\theta_t-\theta_*||_{\mathbf{H}_t(\theta_*)} \leq 2(1+2S)\gamma_t(\delta).
$
\end{lemma}
Proofs of Lemma~\ref{lem: relationship between Gt and Ht main} and~\ref{lem: projection with Gt and Ht relation main} have been deferred to~\ref{sec: self concordance results appendix}. Notice that Lemma~\ref{lem: projection with Gt and Ht relation main} is a key result which characterizes worthiness of the confidence set $C_t(\delta)$. Recall that $\gamma_T(\delta) = \mathrm{O}(\sqrt{d\log(KT)})$ (with a tuned $\lambda$ as in Corollary~\ref{corol: simplified regret result}). Therefore, any $\theta\,\in\,C_t(\delta)$ is not too far from the optimal $\theta_*$ under the norm induced by $||\cdot||_{\mathbf{H}_t(\theta_*)}$. Now, we use Lemma~\ref{lem: projection with Gt and Ht relation main} in Eq~(\ref{eq: prediction error eq main 1}) to get:
\begin{align}\label{eq: prediction error main 3}
&\PredErr(\designQt,\theta_t) \leq
2(1+2S)\gamma_t(\delta)\sumInAssortT|\alpha_i(\designQt,\theta_*,\theta_t)||x_{t,i}||_{\mathbf{H}_t^{-1}(\theta_*)}|.
\end{align}
The quantity $\alpha_i(\designQt,\theta_t,\theta_*)$ as described in the Eq~(\ref{eq: notation alpha main}) is upper bounded in the following result 

\begin{lemma}\label{lem: prediction error single round softmax helper main}
For the assortment chosen by the algorithm \algotwo{}, $\mathcal{Q}_t$ as given by Eq~(\ref{eq: algo decision}) and any $\theta \in C_t(\delta)$ the following holds with probability at least $1-\delta$: $
     \alpha_i(\designQt,\theta_*,\theta)\leq
     \Mmodeldot_i(\designQt^\top\theta_*) + 2(1+2S)M \gamma_t(\delta)||x_{t,i}||_{\mathbf{H}_t^{-1}(\theta_*)}.$
\end{lemma}
We use the result of Lemma~\ref{lem: prediction error single round softmax helper main} in Eq~(\ref{eq: prediction error main 3}) followed by an application of Lemma~\ref{lem: relationship between Gt and Ht main} and the relation $\sumInAssortT||x_{t,i}||^2_{\mathbf{H}_t^{-1}(\theta_t)} \leq \kappa\sumInAssortT||x_{t,i}||^2_{\mathbf{V}_t^{-1}}$ from Assumption~\ref{assm: kappa assumption}, to arrive at the statement of Lemma~\ref{lem: prediction error single round softmax main}.

\subsection{Regret calculation}\label{sec: regret calculation main}
The complete technical work is provided in~\ref{sec: bounds on prediction error apendix} and~\ref{sec: regret calculation appendix}. The key step to retrieve the upper bounds of Theorem~\ref{thm: regret upper bounds main} is to calculate $T$ rounds summation of the prediction error as given in Eq~(\ref{eq: prediction error main 55}). Compared to the previous literature~\cite{filippi2010parametric,li2017provably,oh2021multinomial}, the term $\sumInAssortT\Mmodeldot_i(\designQt^\top\theta_*)||x_{t,i}||_{\mathbf{H}_t^{-1}(\theta_*)}$ is new here. Further, our treatment of this term is much simpler and straight-forward as compared to that in~\cite{faury2020improved}


\subsection{Convex relaxation of the optimization step}\label{sec: convex relaxation main}
Sections~\ref{sec: bounds on prediction error main} \&~\ref{sec: regret calculation main} provide an analytical framework for calculating the regret bounds given: (1) the confidence set (Eq~(\ref{eq: confidence set def main})) with the guarantee that $\theta_* \,\in\,C_t(\delta)$ with probability at least $1-\delta$; (2) the assurance that the confidence set is small (Lemma~\ref{lem: projection with Gt and Ht relation main}). In order to re-use previously developed techniques, we show: (1) $E_t(\delta) \supseteq C_t(\delta)$ (see Eq~(\ref{eq: convex relaxation set prev})) and therefore $\theta_*\,\in\,E_t(\delta)$ with probability at least $1-\delta$ (see Lemma~\ref{lem: convex relaxation lemma main}; (2) an analog of Lemma~\ref{lem: projection with Gt and Ht relation main} using $E_t(\delta)$ (see Lemma~\ref{lem: convex relaxation lemma 2 main}). The proof of Theorem~\ref{thm: regret upper bounds main} is therefore repeated while using Lemma~\ref{lem: convex relaxation lemma 2 main}, following steps as sketched in sections~\ref{sec: bounds on prediction error main} \&~\ref{sec: regret calculation main}. The order dependence of the regret upper bound is retained (see Corollary~\ref{corol: simplified regret result}).

\begin{lemma}\label{lem: convex relaxation lemma main}
$E_t(\delta) \supseteq C_t(\delta)$, therefore for any $\theta \in C_t(\delta)$, we also have $\theta\,\in\,E_t(\delta)$ (see Eq~(\ref{eq: convex relaxation set prev})).
\end{lemma}
The complete proof is provided in~\ref{sec: convex relaxation appendix}. We highlight the usefulness of Lemma~\ref{lem: convex relaxation lemma main}. Since all of set $C_t(\delta)$ lies within $E_t(\delta)$, the consequence of the concentration inequality also implies $\mathbb{P}(\forall t \geq 1, \theta_* \in E_t(\delta)) \geq 1-\delta$.
\begin{lemma}\label{lem: convex relaxation lemma 2 main}
Under the event $\theta_*\,\in\,C_t(\delta)$, the following holds $\forall\,\theta\,\in\,E_t(\delta)$: 
\begin{align*}
    ||\theta-\theta_*||_{\mathbf{H}_t(\theta_*)} \leq (2+2S)\gamma_t(\delta) + 2\sqrt{1+S}\beta_t(\delta).
\end{align*}
When $\lambda_t = \mathrm{O}(d\log(Kt))$, then $\gamma_t(\delta) = \tilde{\mathrm{O}}(\sqrt{d\log(t)})$, $\beta_t(\delta) = \tilde{\mathrm{O}}(\sqrt{d\log(t)})$, and $||\theta - \theta_*||_{\mathbf{H}_t(\theta_*)} = \tilde{\mathrm{O}}(\sqrt{d\log(t)})$.
\end{lemma}
The complete proof can be found in~\ref{sec: convex relaxation appendix}. 


\section{Numerical experiments}\label{sec: numerical experiments}
In this section we compare the empirical performance of our proposed algorithm~\algotwo{} with the previous state of the art in the MNL contextual bandit literature: \algocompUCB{}\citep{oh2021multinomial} and \algocompTS{}\citep{oh2019thompson} on artificial data. We focus on performance comparison for varying values of parameter $\kappa$, and show that our algorithm has a consistently superior performance for different $\kappa$ values in Figure~\ref{fig: main 1}. This highlights the primary contribution of our theoretical analysis. Refer to~\ref{app: numerical experiments} for additional empirical analysis. 

\begin{figure}
\centering
\includegraphics[width=.32\columnwidth]{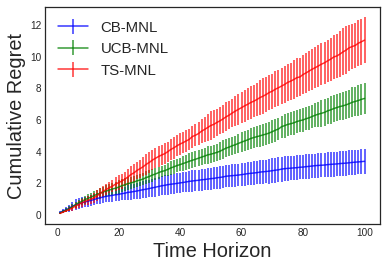}{}
\includegraphics[width=.32\columnwidth]{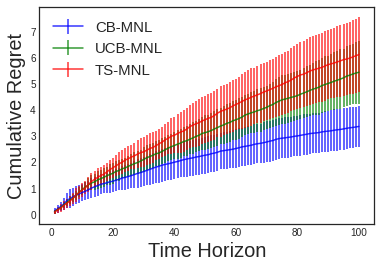}{}
\includegraphics[width=.32\columnwidth]{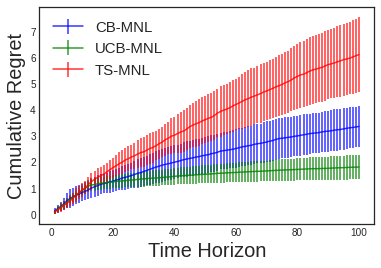}
\caption{Comparison of cumulative regret as a function of time for varying $\kappa$ ( left to right: $\kappa \gg \del{\sqrt{T}}$, $\kappa < \del{\sqrt{T}}$, and $\kappa \ll \del{\sqrt{T}}$) }
\label{fig: main 1}
\end{figure}

For each experimental configuration, we consider a problem instance with inventory size $N=15$, instance dimensions $d=5$, maximum assortment size $K=4$, and time horizon $T=100$, averaged over $25$ Monte Carlo simulation runs. $\theta_*\in \mathbb{R}^d$ is a $d-$dimensional random vector with each coordinate in $[0,1]$, independently and uniformly distributed. The contexts follow a multivariate Gaussian distribution. The $\lambda$ parameter is manually tuned. Algorithm~\algotwo{} only knows the value of $N,T,K,d$. In contrast, algorithms~\algocompTS{}and~\algocompUCB{} also need to know the value of $\kappa$ for their implementation. We observe that that our algorithm~\algotwo{} has robust performance for varying values of $\kappa$.

\section{Conclusion and discussion}\label{sec: conclusion main}
In this work, we proposed an optimistic algorithm for learning under the MNL contextual bandit framework. Using techniques from~\cite{faury2020improved}, we developed an improved technical analysis to deal with the non-linear nature of the MNL reward function. As a result, the leading term in our regret bound does not suffer from the problem-dependent parameter $\kappa$. This contribution is significant as $\kappa$ can be very large (refer to Section~\ref{sec: comparision with other works main}). For example, for $\kappa = \mathrm{O}(\sqrt{T})$, the results of~\cite{oh2021multinomial,oh2019thompson} suffer $\tilde{\mathrm{O}}(T)$ regret, while our algorithm continues to enjoy $\tilde{\mathrm{O}}(\sqrt{T})$. Further, we also presented a tractable version of the decision-making step of the algorithm by constructing a convex relaxation of the confidence set. 

Our result is still $\mathrm{O}(\sqrt{d})$ away from the minimax lower of bound~\cite{chu2011contextual} known for the linear contextual bandit. In the case of logistic bandits,~\cite{li2017provably} makes an i.i.d. assumption on the contexts to bridge the gap (however, they still retain the $\kappa$ factor). Improving the worst-case regret bound by $\mathrm{O}(\sqrt{d})$ while keeping $\kappa$ as an additive term is an open problem. It may be possible to improve the dependence on $\kappa$ by using a higher-order approximation for estimation error. Finding a lower bound on dependence $\kappa$ is an interesting open problem and may require newer techniques than presented in this work.

\cite{oh2019thompson} gave a Thompson sampling (TS) based learning strategy for the MNL contextual bandit. Thompson sampling approaches may not have to search the entire action space to take decisions as optimistic algorithms (such as ours) do. TS-based strategies are likely to have better empirical performance. Authors in~\cite{oh2019thompson} use a confidence set based analysis to bound the estimation error term. However, results in~\cite{oh2019thompson} suffer from the prohibitive scaling of the problem-dependent parameter $\kappa$ that we have overcome here. Modifying our analysis for a TS-based learning strategy could bring together the best of both worlds.

\bibliography{ref}

\onecolumn
\appendix
\section{Appendix}
\subsection{Confidence set}\label{sec: confidence set appendix}
In this section, we justify the design of confidence set defined in Eq~(\ref{eq: confidence set def main}). This particular choice is based on the following concentration inequality for self-normalized vectorial martingales.

\begin{thm}{Appears as Theorem 4 in~\cite{abeille2020instance}}\label{thm: bernstein style concentration}
Let $\{\mathcal{F}_t\}_{t=1}^\infty$ be a filtration. Let $\{x_t\}_{t=1}^{\infty}$ be a stochastic process in $\mathcal{B}_2(d)$ such that $x_t$ is $\mathcal{F}_{t}$ measurable. Let $\{\varepsilon_{t}\}_{t=2}^\infty$ be a martingale difference sequence such that $\varepsilon_{t+1}$ is $\mathcal{F}_{t+1}$ measurable. Furthermore, assume that conditionally on $\mathcal{F}_t$ we have $ \vert \varepsilon_{t+1}\vert \leq 1$ almost surely, and note $\sigma_t^2 \coloneqq \mathbb{E}\left[\varepsilon_{t+1}^2\vert \mathcal{F}_t \right]$. Let $\{\lambda_t\}_{t=1}^{\infty}$ be a predictable sequence of non-negative scalars. Define:
\begin{align*}
\mathbf{H}_t\coloneqq\sum_{s=1}^{t-1}\sigma_s^2 x_sx_s^T + \lambda_t\mathbf{I}_d, \qquad S_t\coloneqq \sum_{s=1}^{t-1} \varepsilon_{s+1}x_s.
\end{align*}
Then for any $\delta\in(0,1]$:
\begin{align*}
\mathbb{P}\Bigg(\exists t\geq 1, \, \left\lVert S_t\right\rVert_{\mathbf{H}_t^{-1}} \!\geq\! \frac{\sqrt{\lambda_t}}{2}\!+\!\frac{2}{\sqrt{\lambda_t}}\log\!\left(\frac{\det\left(\mathbf{H}_t\right)^{\frac{1}{2}}\!\lambda_t^{-\frac{d}{2}}}{\delta}\right)+\frac{2}{\sqrt{\lambda_t}}d\log(2)\Bigg)\leq \delta.
\end{align*}
\end{thm}

Theorem~\ref{thm: bernstein style concentration} cannot be directly used in our setting as in the MNL model the actual rewards (for any time step $s$) $\{r_{s,i}\}_{i \in Q_s}$ are correlated. Hence a concentration almost identical (varying only in minor constant modification) to Theorem~\ref{thm: bernstein style concentration}, appearing as Theorem C.6 in~\cite{perivier2022dynamic} is used instead.

\begin{lemma}[confidence bounds for multinomial logistic rewards]\label{lem: bernstein confidence bound multinomial}
With $\hat{\theta}_t$ as the regularized maximum log-likelihood estimate as defined in Eq~(\ref{eq: MLE estimate}), the following follows with probability at least $1-\delta$:
\begin{align*}
    \forall t \, \geq \, 1,\quad \enVert{g_t(\hat{\theta}_t) - g_t(\theta_*) }_{\mathbf{H}_t^{-1}} \leq \gamma_t\del{\delta},
\end{align*}
where $\mathbf{H}_t\del{\theta_1} = \sum_{s=1}^{t-1}\sumInAssortS\Mmodeldot_i\del{\designQs^\top\theta_1}\xsi \xsi^\top+\lambda\Imatrixd$ and $g_t(\cdot)$ is defined in Eq~(\ref{eq: gt main}).
\end{lemma}
\begin{proof}
$\hat{\theta}_t$ is the maximizer of the regularized log-likelihood:
\begin{align*}
    \mathcal{L}_t^{\lambda_t}(\theta) &= \sum_{s=1}^{t-1}\sum_{i\in \mathcal{Q}_s} r_{s,i}\log\left(\mu_i(\designQs^\top \theta)\right) - \frac{\lambda_t}{2}\enVert{\theta}_2^2,
\end{align*}
where $\mu_i(\designQs^\top\theta)$ is given by Eq~(\ref{eq: rewards eq}) as $\frac{e^{x^\top_{s,i}\theta}}{1+\sumInAssortSj e^{x^\top_{s,j}\theta}}$. Solving for $\nabla_\theta \mathcal{L}_t^{\lambda_t} = 0$, we obtain:
\begin{align*}
    \sum_{s=1}^{t-1}\sumInAssortS \mu_i(\designQs^\top\theta)x_{s,i} +\lambda_t\hat{\theta}_t= \sum_{s=1}^{t-1}\sumInAssortS r_{s,i}x_{s,i}
\end{align*}
This result, combined with the definition of $g_t(\theta_*) = \sum_{s=1}^{t-1}\sumInAssortS \mu_i(\designQs^\top \theta_*)x_{s,i}\\+\lambda_t\theta_*$ yields:
\begin{align*}
g_t(\hat{\theta}_t) - g_t(\theta_*) &= \sum_{s=1}^{t-1}\sumInAssortS  \varepsilon_{s,i}x_{s,i} - \lambda_t \theta_*\\
&= S_{t,K} - \lambda_t\theta_*
\end{align*}
where we denoted $\varepsilon_{s,i}\coloneqq r_{s,i}-\mu_i(\designQs^\top \theta_*)$ for all $s\geq 1$ and $i\in[K]$ and $S_{t,K} \coloneqq \sum_{s=1}^{t-1}\sumInAssortS \varepsilon_{s,i}x_{s,i}$ for all $t\geq 1$. For any $\lambda_t \geq 1$, from the definition of $\mathbf{H}_t(\theta_*)$ it follows that $\mathbf{H}^{-1}_t(\theta_*) \preceq \mathbf{I}_d$. Hence, $\enVert{\lambda_t\theta_*}_{\mathbf{H}^{-1}_t(\theta_*)} \leq \enVert{\lambda_t \theta_*}_2$. Later in the proof of Theorem~\ref{thm: regret upper bounds appendix}, we present our choice of $\lambda_t$ which always ensures $\lambda_t\geq 1$. 
\begin{align}\label{eq: bernstein multinomial 1}
    \enVert{g_t(\hat{\theta}_t) - g_t(\theta_*)}_{\mathbf{H}_t^{-1}(\theta_*)} \leq   \enVert{S_{t,K}}_{\mathbf{H}_t^{-1}(\theta_*)} + \sqrt{\lambda_t} S
\end{align}
Conditioned on the filtration set $\mathcal{F}_{t,i}$ (see Section~\ref{sec: model setting} to review the definition of the filtration set), $\varepsilon_{s,i}$ is a martingale difference is bounded by $1$ as we assume the maximum reward that is accrued at any round is upper bounded by $1$. We calculate for all $s \geq 1$:
\begin{align}
    &\mathbb{E}\left[\varepsilon^2_{s,i} \big\vert \mathcal{F}_{t}\right] = \mathbb{E}\left[\del{r_{s,i}-\mu_i(\designQs^\top \theta_*)}^2 \bigg\vert \mathcal{F}_t\right] \nonumber\\ = & \mathbb{V}\left[r_{s,i}\vert \mathcal{F}_t\right] = \mu_i(\designQs^\top \theta_*) \del{1- \mu_i(\designQs^\top \theta_*)}.
\end{align}
Also from Remark~\ref{lem: softmax derivative}, we have :
\begin{equation*}
    \Mmodeldot_i(\designQs^\top \theta_*) = \mu_i(\designQs^\top \theta_*) \del{1- \mu_i(\designQs^\top \theta_*)}.
\end{equation*}
Therefore setting ${H}_t$  as $ \mathbf{H}_t(\theta_*)= \sum_{s=1}^{t-1}\sumInAssortS \Mmodeldot_i(\designQs^\top \theta_*) \xsi \xsi^\top + \lambda_t\mathbf{I}_d $ and $U_t$ as $S_{t,K}$ we invoke an instance of Theorem C.6 in~\cite{perivier2022dynamic} to obtain:
\begin{align}\label{eq: bernstein multinomial 2}
    1-\delta \leq & \mathbb{P}\left(\forall t\geq 1, \enVert{S_t}_{\mathbf{H}_t^{-1}(\theta_*)}\leq \frac{\sqrt{\lambda_t}}{2} +  \frac{2}{\sqrt{\lambda_t}}\log\left(\frac{2^d\det(\mathbf{H}_t(\theta_*))^{1/2}\lambda_t^{-d/2}}{\delta}\right) \right. \nonumber\\ &+\left.  \frac{2d}{\sqrt{\lambda_t}}\log(2)\right)\nonumber\\
\end{align}

We simplify $\det(\mathbf{H}_t(\theta_*))$, using the fact that the multinomial logistic function is $L$-Lipschitz (see Assumption~\ref{assm: kappa assumption}):
\begin{align*}
    \det(\mathbf{H}_t(\theta_*)) = & \det\left(\sum_{s=1}^{t-1}\sumInAssortS \Mmodeldot_i(\designQs^\top \theta_*) \xsi \xsi^\top + \lambda_t\mathbf{I}_d\right) \\ \leq & L ^d \det\left(\sum_{s=1}^{t-1}\sumInAssortS  \xsi \xsi^\top + \frac{\lambda_t}{L }\mathbf{I}_d\right).
\end{align*}
Further, using Lemma~\ref{lem: determinant inequality} and using $\enVert{x_{s,i}}_2\leq 1$ we write:
\begin{align*}
L ^d \det\left(\sum_{s=1}^{t-1}\sumInAssortS  \xsi \xsi^\top + \frac{\lambda_t}{L }\mathbf{I}_d\right) \leq \left(\lambda_t + \frac{ L Kt}{d}\right)^d.
\end{align*}

This we simplify Eq~(\ref{eq: bernstein multinomial 2}) as:
\begin{align}
  1-\delta\leq &\mathbb{P}\left(\forall t\geq 1, \enVert{S_t}_{\mathbf{H}_t^{-1}(\theta_*)}\leq \frac{\sqrt{\lambda_t}}{2} +  \frac{2}{\sqrt{\lambda_t}}\log\left(\frac{\left(\lambda_t +L  Kt/d\right)^{d/2}\lambda_t^{-d/2}}{\delta}\right) \right.\nonumber \\ &+ \left. \frac{2d}{\sqrt{\lambda_t}}\log(2)\right)\nonumber\\
  &\leq \mathbb{P}\left(\forall t\geq 1, \enVert{S_t}_{\mathbf{H}_t^{-1}(\theta_*)}\leq \frac{\sqrt{\lambda_t}}{2} +  \frac{2}{\sqrt{\lambda_t}}\log\left(\frac{\left(1 +\frac{ L K t}{\lambda_t d}\right)^{d/2}}{\delta}\right) \right.\nonumber \\ &+ \left. \frac{2d}{\sqrt{\lambda_t}}\log(2)\right)\nonumber
  \\
  &= \mathbb{P}\left(\forall t\geq 1, \enVert{S_t}_{\mathbf{H}_t^{-1}(\theta_*)}\leq \gamma_t(\delta)-\sqrt{\lambda_t}S \right)\label{eq: bernstein multinomial 3}
\end{align}
Combining Eq~(\eqref{eq: bernstein multinomial 1}) and Eq~(\eqref{eq: bernstein multinomial 3}) yields:
\begin{align*}
    &\mathbb{P}\left(\forall t\geq 1, \,\enVert{g_t(\hat{\theta}_t) - g_t(\theta_*)}_{\mathbf{H}_t^{-1}(\theta_*)} \leq \gamma_t(\delta)\right) \\
    &\geq \mathbb{P}\left( \forall t\geq 1, \,\enVert{S_t}_{\mathbf{H}_t^{-1}(\theta_*)} + \sqrt{\lambda_t} S\leq \gamma_t(\delta)\right)\\
    &\geq  1-\delta.
\end{align*}
This completes the proof.
\end{proof}

It is insightful to compare Theorem~\ref{thm: bernstein style concentration} with Theorem 1 of~\cite{abbasi2011improved}. The later is re-stated below:

\begin{thm}\label{thm: gaussian concentration}
Let $\{\mathcal{F}\}_{t=0}^{\infty}$ be a filtration. Let $\{\eta\}_{t=1}^{\infty}$ be a real-valued stochastic process such that $\eta_t$ is $\mathcal{F}_t$-measurable and $\eta_t$ is conditionally $R$-sub-Gaussian for some $R\geq0$, i.e
$$
\forall\, \lambda_t\,\in\mathbb{R},\qquad \mathbb{E}\sbr{\exp(\lambda_t\eta_t)\mid \mathcal{F}_{t-1}}\leq \exp\del{\frac{\lambda_t^2R^2}{2}}.
$$
Let $\{x_t\}_{t=1}^{\infty}$ be an $\mathbb{R}^d-$valued stochastic process such that $X_t$ is $\mathcal{F}_{t-1}$-measurable. Assume $\mathbf{V}$ is a $d \times d$ positive definite matrix. For any $t\geq 0$, define:
$$
\overline{\mathbf{V}}_t = \mathbf{V} +\sum_{s=1}^t x_s x_s^\top, \qquad \quad S_t=\sum_{s=1}^t \eta_s x_s.
$$
Then, for any $\delta >0$, with probability at least $1-\delta$. for all $t\geq 0$,
$$
||S_t||_{\overline{\mathbf{V}}_t^{-1}} \leq 2R \log\del{\frac{\det(\overline{\mathbf{V}})^{-\nicefrac{1}{2}} \det(\mathbf{V})^{-\nicefrac{1}{2}}}{\delta}}.
$$
\end{thm}
Theorem~\ref{thm: gaussian concentration} makes an uniform sub-Gaussian assumption and unlike Theorem~\ref{thm: bernstein style concentration} does not take into account local variance information.

\subsection{Local information preserving norm}\label{sec: Local Information Preserving Norm appendix} 
Deviating from the previous analyses as in~\cite{filippi2010parametric,li2017provably}, we describe norm which preserves the local information
The matrix $\designQs$ is the design matrix composed of the contexts $x_{s,1},x_{s,2},\cdots,x_{s,K}$ received at time step $s$ as its columns. The expected reward due to the $i_{th}$ item in the assortment is given by:
$$\mu_i(\designQs^\top\theta) = \frac{e^{x^\top_{s,i}\theta}}{1+\sumInAssortSj e^{x^\top_{s,j}\theta}}.
$$
Further, we consider the following integral:
\begin{align}\label{eq: line integral multinomial}
    \int_{\nu=0}^1 \Mmodeldot_i\del{\nu \designQs^\top\theta_2+\del{1-\nu}\designQs^\top\theta_1}\cdot d\nu 
    & = \int_{x_{s,i}^\top\theta_1}^{x_{s,i}^\top\theta_2}\frac{1}{x_{s,i}^\top(\theta_2-\theta_1)} \Mmodeldot_i(t_i)\cdot dt_i ,
\end{align}
where $\Mmodeldot_i$ is the partial derivative of $\Mmodel_i$ in the direction of the $i_{th}$ component and $\int_{x_{s,i}^\top\theta_1}^{x_{s,i}^\top\theta_2} \Mmodeldot_i(t_i)\cdot dt_i$ represents integration of $\Mmodeldot(\cdot)$ with respect to the coordinate $t_i$ ( hence the limits of the integration only consider change in the coordinate $t_i$). For notation purposes which would become clear later, we define:
\begin{align}\label{eq: gt relation helper 2}
    \alpha_i(\designQs,\theta_1,\theta_2)x_{s,i}^\top(\theta_2-\theta_1) 
    &\coloneqq \mu_{i}(\designQs^\top\theta_2)-\mu_{i}(\designQs^\top\theta_1)\nonumber\\
    &=\frac{e^{x_{s,i}^\top\theta_2}}{1+\sumInAssortSj e^{x_{s,j}^\top\theta_2}} - \frac{e^{x_{s,i}^\top\theta_1}}{1+\sumInAssortSj e^{x_{s,j}^\top\theta_1}}\\
    &= \int_{x_{s,i}^\top\theta_1}^{x_{s,i}^\top\theta_2} \Mmodeldot_i(t_i)\cdot dt_i\nonumber,
\end{align}
where the second step is due to Fundamental Theorem of Calculus. We have exploited the two ways to view the multinomial logit function: sum of individual probabilities and a vector valued function. We write:
\begin{equation}\label{eq: gt relation helper 1}
\sumInAssortS\alpha_i(\designQs,\theta_1,\theta_2)x_{s,i}^\top(\theta_2-\theta_1) = \sumInAssortS\int_{\nu=0}^1 \Mmodeldot_i\del{\nu \designQs^\top\theta_2+\del{1-\nu}\designQs^\top\theta_1}\cdot d\nu 
\end{equation}
We also have:
\begin{align}\label{eq: FTC multinonmial}
    \Mmodel(\designQs^\top\theta_1)- \Mmodel(\designQs^\top\theta_2) = \sum_{i=1}^K\alpha_i(\designQs,\theta_2,\theta_1)x_{s,i}^\top(\theta_1-\theta_2).
\end{align}
It follows that:
\begin{align*}
    g(\theta_1) -g(\theta_2) =&
    \sum_{s=1}^{t-1}\sumInAssortS \del{\frac{e^{x_{s,i}^\top\theta_1}}{1+\sumInAssortSj e^{x_{s,j}^\top\theta_1}} -  \frac{e^{x_{s,i}^\top\theta_2}}{1+\sumInAssortSj e^{x_{s,j}^\top\theta_2}}}x_{s,i} \nonumber \\ &+ \lambda_t(\theta_1-\theta_2)\\
    =&\sum_{s=1}^{t-1}\sumInAssortS \alpha_i(\designQs,\theta_2,\theta_1)x_sx_x^\top(\theta_1-\theta_2) +\lambda_t(\theta_1-\theta_2)\\
    =& \mathbf{G}_t(\theta_2,\theta_1)(\theta_1-\theta_2),
\end{align*}
where $\mathbf{G}_t\del{\theta_1,\theta_2} \coloneqq \sum_{s=1}^{t-1}\sumInAssortS \alpha_i\del{\designQs,\theta_1,\theta_2}x_sx_s^\top+\lambda_t\Imatrixd$. Since $\alpha\del{\designQs,\theta_1,\theta_2} \geq \frac{1}{\kappa}$ (from Assumption~\ref{assm: kappa assumption}), therefore $\mathbf{G}_t(\theta_1,\theta_2) \succ \mathbf{O}_{d\times d}$. Hence we get:
\begin{equation}\label{eq: theta and G relation multinomial}
\enVert{\theta_1-\theta_2}_{\mathbf{G}_t(\theta_2,\theta_1)} = \enVert{g(\theta_1)-g(\theta_2)}_{\mathbf{G}_t^{-1}(\theta_2,\theta_1)}.
\end{equation}

\subsection{Self-Concordance Style Relations for Multinomial Logistic Function}\label{sec: self concordance results appendix}

\begin{lemma}\label{lem: application of self concordance softmax}
For an assortment $\mathcal{Q}_s$ and $\theta_1,\theta_2\,\in\,\Theta$, the following holds:
\begin{align*}
\sumInAssortS \alpha_i(\designQs,\theta_2,\theta_1) &=\sumInAssortS \int_{\nu=0}^1 \Mmodeldot_i\del{\nu \designQs^\top\theta_2+\del{1-\nu}\designQs^\top\theta_1}\cdot d\nu \\
&\geq \sumInAssortS \Mmodeldot_i(\designQs^\top \theta_1)\del{1+|x_{s,i}^\top\theta_1-x_{s,i}^\top\theta_2|}^{-1}
\end{align*}
\end{lemma}
\begin{proof}
We write:
\begin{align}\label{eq: application of self con 5}
    \int_{\nu=0}^1 \Mmodeldot_i\del{\nu \designQs^\top\theta_2+\del{1-\nu}\designQs^\top\theta_1}\cdot d\nu  
    & = \sumInAssortS \int_{x_{s,i}^\top\theta_1}^{x_{s,i}^\top\theta_2}\frac{1}{x_{s,i}^\top(\theta_2-\theta_1)} \Mmodeldot_i(t_i)\cdot dt_i,
\end{align}
where $\int_{x_{s,i}^\top\theta_1}^{x_{s,i}^\top\theta_2} \Mmodeldot_i(t_i)\cdot dt_i$ represents integration of $\Mmodeldot(\cdot)$ with respect to the coordinate $t_i$ ( hence the limits of the integration only consider change in the coordinate $t_i$). For some $z>z_1\,\in\,\mathbb{R}$, consider:
\begin{align*}
    \int_{z_1}^{z} \frac{d}{dt_i}\log\del{\dot{\mu}_i(t_i)}\cdot dt_i = \int_{z_1}^{z} \frac{\nabla^2\mu_{i,i}(t_i)}{\Mmodeldot_i(t_i)} dt_i,
\end{align*}
where $\nabla^2\mu_{i,i}(\cdot)$ is the double derivative of $\Mmodel(\cdot)$. Using Lemma~\ref{lem: self concordance property softmax}, we have $-1 \leq\frac{\nabla^2\mu_{i,i}(\cdot)}{\Mmodeldot_i(\cdot)}\leq 1$. Thus we get:
\begin{equation*}
    -(z-z_1)\leq\int_{z_1}^{z} \frac{d}{dt_i}\log\del{\dot{\mu}_i(t_i)}.dt_i\leq (z-z_1)
\end{equation*}
Using Fundamental Theorem of Calculus, we get:
\begin{align}\label{eq: application of self con 1}
    &-(z-z_1)\leq\log\del{\dot{\mu}_i(z)}- \log\del{\dot{\mu}_i(z_1)}\leq (z-z_1)\nonumber\\
    \therefore&~\dot{\mu}_i(z_1)\exp(-(z-z_1))\leq \dot{\mu}_i(z) \leq \dot{\mu}_i(z_1)\exp(z-z_1)
\end{align}
Using Eq~(\ref{eq: application of self con 1}) and for $z_2\geq z_1\,\in\,\mathbb{R}$, and for all $i\,\in\,[K]$ and we have:
\begin{align}\label{eq: application of self con 2}
    &\Mmodeldot_i(z_1) \del{1-\exp(-(z_2-z_1))} \leq \int_{z_1}^{z_2} \Mmodeldot(t_i)dt_i \leq \Mmodeldot_i(z_1)\del{\exp(z_2-z_1)-1}\nonumber\\
    \therefore&~\Mmodeldot_i(z_1) \frac{1-\exp(-(z_2-z_1))}{z_2-z_1} \leq \frac{1}{z_2-z_1}\int_{z_1}^{z_2} \Mmodeldot(t_i)dt_i \leq \Mmodeldot_i(z_1)\frac{\exp(z_2-z_1)-1}{z_2-z_1}.
\end{align}

Reversing the role of $z_1$ and $z_2$, such that $z_2\leq z_1$ then again by using Eq~(\ref{eq: application of self con 1}) we write:

\begin{align}\label{eq: application of self con 3}
    \Mmodeldot_i(z_1) \frac{\exp(-(z_1-z_2))-1}{z_2-z_1} \leq \frac{1}{z_2-z_1}\int_{z_1}^{z_2} \Mmodeldot(t_i)dt_i \leq \Mmodeldot_i(z_1)\frac{\exp(z_1-z_2)-1}{z_2-z_1}.
\end{align}

Combining Eq~(\ref{eq: application of self con 2}) and~(\ref{eq: application of self con 3}) and for all $i\,\in\,[K]$ we get:

\begin{align}\label{eq: application of self con 4}
    \Mmodeldot_i(z_1) \frac{1-\exp(-|z_1-z_2|)}{|z_1-z_2|} \leq \frac{1}{z_2-z_1}\int_{z_1}^{z_2} \Mmodeldot(t_i)dt_i.
\end{align}
If $x\geq 0$, then $e^{-x}\leq (1+x)^{-1}$, and therefore $(1-e^{-x})/x \geq (1+x)^{-1}$. Thus we lower bound the left hand side of Eq~(\ref{eq: application of self con 4}) as:
\begin{align*}
    \Mmodeldot_i(z_1)\del{1+|z_1-z_2|}^{-1} \leq \Mmodeldot_i(z_1) \frac{1-\exp(-|z_1-z_2|)}{|z_1-z_2|} \leq \frac{1}{z_2-z_1}\int_{z_1}^{z_2} \Mmodeldot(t_i)dt_i.
\end{align*}
Using above with $z_2 =x_{s,i}^\top\theta_2$ and $z_1 = x_{s,i}^\top\theta_1$ in Eq~(\ref{eq: application of self con 5}) gives:
\begin{align*}
    &\sumInAssortS \int_{\nu=0}^1 \Mmodeldot_i\del{\nu \designQs^\top\theta_2+\del{1-\nu}\designQs^\top\theta_1}\cdot d\nu \\ =& 
     \sumInAssortS \int_{x_{s,i}^\top\theta_1}^{x_{s,i}^\top\theta_2}\frac{1}{x_{s,i}^\top(\theta_2-\theta_1)} \Mmodeldot_i(t_i)\cdot dt_i
      \geq  \sumInAssortS \Mmodeldot_i(\designQs^\top \theta_1)\del{1+|x_{s,i}^\top\theta_1-x_{s,i}^\top\theta_2|}^{-1}.
\end{align*}

\end{proof}

\begin{customlemma}{4}\label{lem: relationship between Gt and Ht}
For all $\theta_1,\theta_2\,\in\,\Theta$ such that $S\coloneqq\max_{\theta\,\in\,\Theta}\enVert{\theta}_2$ (Assumption~\ref{assm: bounded parameters assumption}), the following inequalities hold:
\begin{align*}
    & \mathbf{G}_t(\theta_1,\theta_2) \succeq (1+2S)^{-1}\mathbf{H}_t(\theta_1)\\
    & \mathbf{G}_t(\theta_1,\theta_2) \succeq (1+2S)^{-1}\mathbf{H}_t(\theta_2)
\end{align*}
\end{customlemma}
\begin{proof}
From Lemma~\ref{lem: application of self concordance softmax}, we have:
\begin{align}
\sumInAssortS \alpha_i(\designQs,\theta_2,\theta_1)
&\geq \sumInAssortS \del{1+|x_{s,i}^\top\theta_1-x_{s,i}^\top\theta_2|}^{-1}\Mmodeldot_i(\designQs^\top \theta_1)\nonumber\\
&\geq \sumInAssortS \left(1+\enVert{x_{s,i}}_2\enVert{\theta_1-\theta_2}_2\right)^{-1}\Mmodeldot_i(\designQs^\top \theta_1)\tag{Cauchy-Schwartz}\\
&\geq \sumInAssortS \left(1+2S\right)^{-1}\Mmodeldot_i(\designQs^\top \theta_1) \tag{$\theta_1,\theta_2\in\Theta,\, ||x_{s,i}||_2\leq1$}
\end{align}
Now we write $\mathbf{G}_t(\theta_1,\theta_2)$ as:
\begin{align*}
    \mathbf{G}_t(\theta_1,\theta_2) &= \sum_{s=1}^{t-1} \sumInAssortS \alpha_i(\designQs,\theta_2,\theta_1)\xsi \xsi^\top  + \lambda_t\mathbf{I}_d\\ 
    &\succeq (1+2S)^{-1}\sum_{s=1}^{t-1}\sumInAssortS \Mmodeldot_i(\designQs^\top \theta_1)\xsi \xsi^\top  + \lambda_t\mathbf{I}_d\\
    &= (1+2S)^{-1}\del{\sum_{s=1}^{t-1}\sumInAssortS \Mmodeldot_i(\designQs^\top \theta_1)\xsi \xsi^\top  + (1+2S)\lambda_t\mathbf{I}_d}\\
    &\succeq (1+2S)^{-1}\del{\sum_{s=1}^{t-1}\sumInAssortS \Mmodeldot_i(\designQs^\top \theta_1)\xsi \xsi^\top  + \lambda_t\mathbf{I}_d}\\
    &=(1+2S)^{-1}\mathbf{H}_t(\theta_1).
\end{align*}
Since, $\theta_1$ and $\theta_2$ have symmetric roles in the definition of $\alpha_i(\designQs,\theta_2,\theta_1)$, we also obtain the second relation by a change of variable directly.
\end{proof}

The following Lemma presents a crucial bound over the deviation $(\theta-\theta_*)$, which we extensively use in our derivations.
\begin{customlemma}{5}\label{lem: projection with Gt and Ht relation}
For $\theta \in C_t\del{\delta} $, we have the following relation with probability at least $1-\delta$:
\begin{align}
&\enVert{\theta-\theta_*}_{\mathbf{H}_t(\theta)} \leq 2(1+2S) \gamma_t(\delta).\label{eq: lemma sub1}
\end{align}
\end{customlemma}
\begin{proof}
Since $\theta,\theta_* \in \Theta$, then by Lemma~\ref{lem: relationship between Gt and Ht}, it follows that:
\begin{equation*}
    \enVert{\theta-\theta_*}_{\mathbf{H}_t(\theta)} \leq \sqrt{1+2S}\enVert{\theta-\theta_*}_{\mathbf{G}_t(\theta,\theta_*)}.
\end{equation*}
From triangle inequality, we write :
\begin{align*}
    \enVert{g (\theta_*)-g (\theta)}_{\mathbf{G}_t^{-1}(\theta,\theta_*)} \leq \enVert{g (\theta_*)-g (\hat{\theta}_t)}_{\mathbf{G}_t^{-1}(\theta,\theta_*)} +  \enVert{g (\hat{\theta}_t)-g (\theta)}_{\mathbf{G}_t^{-1}(\theta,\theta_*)},
\end{align*}
where $\hat{\theta}_t$ is the MLE estimate. Further Lemma~\ref{lem: relationship between Gt and Ht} gives:
\begin{align*}
    \enVert{g (\theta_*)-g (\theta)}_{\mathbf{G}_t^{-1}(\theta,\theta_*)} &\leq \sqrt{1+2S}\enVert{g (\theta_*)-g (\hat{\theta}_t)}_{\mathbf{H}_t^{-1}(\theta_*)}\\&+\sqrt{1+2S}\enVert{g (\hat{\theta}_t)-g (\theta)}_{\mathbf{H}_t^{-1}(\theta)}.
\end{align*}
since $\theta$ is the minimizer of $\enVert{g (\theta)-g (\hat{\theta}_t)}_{\mathbf{H}_t^{-1}(\theta)}$, therefore we write:
\begin{align*}
    \enVert{g (\theta_*)-g (\theta)}_{\mathbf{G}_t^{-1}(\theta,\theta_*)} \leq 2\sqrt{1+2S}\enVert{g (\theta_*)-g (\hat{\theta}_t)}_{\mathbf{H}_t^{-1}(\theta_*)}.
\end{align*}
Finally, the Eq~(\ref{eq: lemma sub1}) follows by an application of Lemma~\ref{lem: bernstein confidence bound multinomial} as:
\begin{align*}
    \enVert{g (\theta_*)-g (\theta)}_{\mathbf{H}_t^{-1}(\theta_*)} &\leq \gamma_t(\delta).
\end{align*}

\end{proof}

\subsection{Bounds on prediction error}\label{sec: bounds on prediction error apendix}
\begin{customlemma}{6}\label{lem: prediction error single round softmax helper}
For the assortment chosen by the algorithm \algotwo{}, $\mathcal{Q}_t$ as given by Eq~(\ref{eq: algo decision}) and any $\theta \in C_t(\delta)$ the following holds with probability at least $1-\delta$:
    \begin{align*}
     \alpha_i(\designQt,\theta_*,\theta)&\leq
     \Mmodeldot_i\del{\designQt^\top\theta_*} + 2(1+2S)M \gamma_t(\delta)\enVert{x_{t,i}}_{\mathbf{H}_t^{-1}(\theta_*)}.
\end{align*}
\end{customlemma}
\begin{proof}
Consider the mulinomial logit function:
\begin{align}\label{eq: prediction error helper 2}
    \alpha_i(\designQt,\theta_*,\theta)x_{t,i}^\top(\theta-\theta_*) =\frac{e^{x_{t,i}^\top\theta}}{1+\sumInAssortTj e^{x_{t,j}^\top\theta}} - \frac{e^{x_{t,i}^\top\theta_*}}{1+\sumInAssortTj e^{x_{t,j}^\top\theta_*}}. 
\end{align}
We use second-order Taylor expansion for each component of the multinomial logit function at $a_i$. Consider for all $i\,\in\,[K]$:
\begin{align}\label{eq: prediction error helper 1}
    f_i(r_i) &= \frac{e^{r_i}}{1+e^{r_i}+\sum_{j\in\mathcal{Q}_s,j\neq i}e^{r_j}}\nonumber\\
    & \leq f(a_i) + f_i'(a_i)(r_i-a_i) + \frac{f_i''(a_i)(r_i-a_i)^2}{2}.
\end{align}
In Eq~(\ref{eq: prediction error helper 1}), we substitute: $f_i(\cdot) \to \Mmodel_i$, $r_i\to x_{t,i}^\top\theta$, and $a_i \to x_{t,i}^\top\theta_*$. Thus we re-write Eq~(\ref{eq: prediction error helper 2}) as:
\begin{align*}
    \alpha_i(\designQt,\theta_*,\theta)x_{s,i}^\top(\theta-\theta_*) &\leq \Mmodeldot_i\del{\designQt^\top\theta_*}(x_{t,i}^\top(\theta-\theta_*)) + \ddot{\mu}_i\del{\designQt^\top\theta_*}(x_{t,i}^\top(\theta-\theta_*)^2,\\
    \therefore ~ ~\alpha_i(\designQt,\theta_*,\theta)&\leq \Mmodeldot_i\del{\designQt^\top\theta_*}(x_{t,i}^\top(\theta-\theta_*)) + \ddot{\mu}_i\del{\designQt^\top\theta_*}|x_{t,i}^\top(\theta-\theta_*)|\\
    &\leq \Mmodeldot_i\del{\designQt^\top\theta_*} + M \envert{x_{t,i}^\top(\theta-\theta_*)},
\end{align*}
where we upper bound $\ddot{\mu}_i$ by $M$. An application of Cauchy-Schwarz gives us:
\begin{align}
    \envert{x_{t,i}^\top(\theta_*-\theta)} &\leq \enVert{x_{t,i}}_{\mathbf{H}_t^{-1}(\theta_*)}\enVert{\theta_*-\theta}_{\mathbf{H}_t(\theta_*)}
    \end{align}

Upon Combining the last two equations we get:
\begin{align*}
     \alpha_i(\designQt,\theta_*,\theta)&\leq
     \Mmodeldot_i\del{\designQt^\top\theta_*} + \enVert{x_{t,i}}_{\mathbf{H}_t^{-1}(\theta_*)}\enVert{\theta_*-\theta}_{\mathbf{H}_t(\theta_*)}.
\end{align*}
From Lemma~\ref{lem: projection with Gt and Ht relation} we get:
\begin{align*}
     \alpha_i(\designQt,\theta_*,\theta)&\leq
     \Mmodeldot_i\del{\designQt^\top\theta_*} + 2(1+2S)M \gamma_t(\delta)\enVert{x_{t,i}}_{\mathbf{H}_t^{-1}(\theta_*)}.
\end{align*}

\end{proof}

\begin{customlemma}{3}\label{lem: prediction error single round softmax}
For the assortment chosen by the algorithm \algotwo{}, $\mathcal{Q}_t$ as given by Eq~(\ref{eq: algo decision}) and any $\theta \in C_t(\delta)$ the following holds with probability at least $1-\delta$:
\begin{align*}
\PredErr(\designQt,\theta) \leq & \del{2+4S}\gamma_t(\delta)\sumInAssortT \Mmodeldot_i\del{\designQt^\top\theta_*}\enVert{x_{t,i}}_{\mathbf{H}_t^{-1}(\theta_*)} \\ &+ 4\kappa(1+2S)^2M \gamma_t(\delta)^2\sumInAssortT \enVert{x_{t,i}}^2_{\mathbf{V}_t^{-1}}
\end{align*}
\end{customlemma}
\begin{proof}

\begin{align}
\PredErr(\designQt,\theta)  &=~\envert{ \Mmodel(\designQt^\top \theta)-\Mmodel(\designQt^\top \theta_*)}\nonumber\\
=&~\envert{\sumInAssortT \alpha_i(\designQt,\theta_*,\theta)x_{t,i}^\top(\theta-\theta_*)}\tag{From Eq~(\ref{eq: FTC multinonmial})}\\
\leq&~\envert{\sumInAssortT \alpha_i(\designQt,\theta_*,\theta)\enVert{x_{t,i}}_{\mathbf{H}_t^{-1}(\theta_*)}\enVert{\theta_*-\theta}_{\mathbf{H}_t(\theta_*)}}\tag{Cauchy-Schwarz inequality and Eq~(\ref{eq: theta and G relation multinomial})}\\
\leq&~2(1+2S)\gamma_t(\delta)\sumInAssortT \envert{\alpha_i(\designQt,\theta_*,\theta)\enVert{x_{t,i}}_{\mathbf{H}_t^{-1}(\theta_*)}}\tag{From Lemma~\ref{lem: projection with Gt and Ht relation}}\\
\leq&~2(1+2S)\gamma_t(\delta)\sumInAssortT \left(\Mmodeldot_i\del{\designQt^\top\theta_*}\enVert{x_{t,i}}_{\mathbf{H}_t^{-1}(\theta_*)}\right. \nonumber \\ &+ \left. 2(1+2S) M \gamma_t(\delta)\enVert{x_{t,i}}^2_{\mathbf{H}_t^{-1}(\theta_*)}\right)\tag{From Lemma~\ref{lem: prediction error single round softmax helper}}\\
\end{align}

Upon re-arranging the terms we get:
\begin{align*}
\PredErr(\designQt,\theta) 
\leq & \del{2+4S}\gamma_t(\delta)\sumInAssortT \Mmodeldot_i\del{\designQt^\top\theta_*}\enVert{x_{t,i}}_{\mathbf{H}_t^{-1}(\theta_*)} \\ &+ 4\kappa(1+2S)^2M \gamma_t(\delta)^2\sumInAssortT \enVert{x_{t,i}}^2_{\mathbf{V}_t^{-1}},
\end{align*}
where we use $\mathbf{H}_t^{-1}(\theta_*) \succeq \kappa^{-1}\mathbf{V}_t$ from Assumption~\ref{assm: kappa assumption}. 
\end{proof}

\begin{customcorollary}{7}\label{lem: prediction error sum softmax}
For the assortment chosen by the algorithm \algotwo{}, $\mathcal{Q}_t$ as given by Eq~(\ref{eq: algo decision}) and any $\theta \in C_t(\delta)$ the following holds with probability at least $1-\delta$:
\begin{align*}
\PredErr(\designQt,\theta) \leq&  2\del{1+2S}\gamma_t(\delta)\sumInAssortT \enVert{\tilde{x}_{t,i}}_{\mathbf{J}_t^{-1}} \\ &+ 4\kappa(1+2S)^2M \gamma_t(\delta)^2\sumInAssortT \enVert{x_{t,i}}^2_{\mathbf{V}_t^{-1}},
\end{align*}
where $\tilde{x}_{t,i} = \sqrt{\Mmodeldot_i(\designQt^\top\theta_*)}x_{t,i}$ and $\enVert{x}_{\mathbf{H}^{-1}_t(\theta_*)}= \enVert{x}_{\mathbf{J}^{-1}_t}$.
\end{customcorollary}
\begin{proof}
This directly follows from the uniqueness and realizability of $\theta_*$.

\end{proof}

\subsection{Regret calculation}\label{sec: regret calculation appendix}
The following two lemmas give the upper bounds on the self-normalized
vector summations.
\begin{lemma}\label{lem: sums of norm square Jm}
\begin{align*}
    \sum_{t=1}^T\min\cbr{\sumInAssortT \enVert{\tilde{x}_{t,i}}^2_{\mathbf{J}^{-1}_{T+1}(\theta)},1} &\leq~2d\log\del{1+\frac{L KT}{d\lambda_t}}.
\end{align*}
\end{lemma}
\begin{proof}
The proof follows by a direct application of Lemma~\ref{lem: elliptical potential} and~\ref{lem: determinant inequality} as:
\begin{align}
    &\sum_{t=1}^T\min\cbr{\sumInAssortT \enVert{\tilde{x}_{t,i}}^2_{\mathbf{J}^{-1}_{T+1}(\theta)},1}\nonumber\\
    \leq& ~2\log\del{\frac{\det(\mathbf{J}_{T+1})}{\lambda_t^d}}\tag{From Lemma~\ref{lem: elliptical potential}}\\
    =&~2\log\del{\frac{\det\del{\sum_{s=1}^{t-1}\sumInAssortS \Mmodeldot_i(\designQt^\top\theta_*)\xti \xti^\top+\lambda_t\Imatrixd}}{\lambda_t^d}}\nonumber\\
    \leq&~2\log\del{\frac{\det\del{\sum_{s=1}^{t-1}\sumInAssortS L \xti \xti^\top+\lambda_t\Imatrixd}}{\lambda_t^d}}\tag{Upper bound by Lipschitz constant}\\
    \leq&~2\log\del{\frac{L ^d\det\del{\sum_{s=1}^{t-1}\sumInAssortS \xti \xti^\top+\nicefrac{\lambda_t}{L }\Imatrixd}}{\lambda_t^d}}\nonumber\\
    \leq&~2\log\del{\frac{\det\del{\sum_{s=1}^{t-1}\sumInAssortS L \xti \xti^\top+\lambda_t\Imatrixd}}{\lambda_t^d}}\nonumber\\
    \leq&~2d\log\del{1+\frac{L KT}{d\lambda_t}}.\tag{From Lemma~\ref{lem: determinant inequality}}
\end{align}

\end{proof}

Similar to Lemma~\ref{lem: sums of norm square Jm}, we prove the following.
\begin{lemma}\label{lem: sums of norm square Vm}
\begin{align*}
    \sum_{t=1}^T\min\cbr{\sumInAssortT \enVert{x_{t,i}}^2_{\mathbf{V}^{-1}_{T+1}(\theta)},1} &\leq~2d\log\del{1+\frac{KT}{d\lambda_t}}.
\end{align*}
\end{lemma}
\begin{proof}
\begin{align}
   &\sum_{t=1}^T\min\cbr{\sumInAssortT \enVert{x_{t,i}}^2_{\mathbf{V}^{-1}_{T+1}(\theta)},1}\nonumber\\ \leq& ~2\log\del{\frac{\det(\mathbf{V}_{T+1})}{\lambda_t^d}}\tag{From Lemma~\ref{lem: elliptical potential}, set $\MmodeldotUnder(\cdot)=1$}\\
    =&~2\log\del{\frac{\det\del{\sum_{s=1}^{t-1}\sumInAssortS \xti \xti^\top+\lambda_t\Imatrixd}}{\lambda_t^d}}\nonumber\\
    \leq&~2d\log\del{1+\frac{KT}{d\lambda_t}}.\tag{From Lemma~\ref{lem: determinant inequality}}
\end{align}
\end{proof}

\begin{customthm}{1}\label{thm: regret upper bounds appendix}
With probability at least $1-\delta$:
\begin{align*}
    \mathbf{R}_T \leq& C_1\gamma_t(\delta)\sqrt{2d\log\del{1+\frac{L KT}{d\lambda_t}}T}+ C_2\kappa \gamma_t(\delta)^2d\log\del{1+\frac{KT}{d\lambda_t}},
\end{align*}
where the constants are given as $C_1 = \del{4+8S}$, $C_2= 4(4+8S)^{\nicefrac{3}{2}}M$ and $\gamma_t(\delta)$ is given by Eq~(\ref{eq: gamma value main}).
\end{customthm}
\begin{proof}
The regret is upper bounded by the prediction error.
\begin{align}
    \mathbf{R}_T \leq & \sum_{t=1}^T\min\cbr{\PredErr\del{\designQt,\MthetaProjectedt},1}\tag{$R_{\max}=1$}\\
    \leq & \sum_{t=1}^T\min\left\{\del{2+4S}\gamma_t(\delta)\sumInAssortT \enVert{x_{t,i}}_{\mathbf{J}_t^{-1}} \right.\nonumber \\ &+ \left. 8\kappa(1+2S)^2M \gamma_t(\delta)^2\sumInAssortT \enVert{x_{t,i}}^2_{\mathbf{V}_t^{-1}} ,1\right\}\tag{From Lemma~\ref{lem: prediction error sum softmax}}\\
    \leq & 2\del{1+2S}\gamma_t(\delta)\sum_{t=1}^T\min\cbr{\sumInAssortT \enVert{x_{t,i}}_{\mathbf{J}_t^{-1}} ,1} \nonumber \\ &+  8(1+2S)^2\kappa M \gamma_t(\delta)^2\sum_{t=1}^T\min\cbr{\sumInAssortT \enVert{x_{t,i}}^2_{\mathbf{V}_t^{-1}} ,1}\nonumber\\
    \leq & 2\del{1+2S}\gamma_t(\delta)\sqrt{T}\sqrt{\sum_{t=1}^T\min\cbr{\sumInAssortT \enVert{x_{t,i}}^2_{\mathbf{J}_t^{-1}} ,1}} \nonumber \\ &+ 8(1+2S)^2\kappa M \gamma_t(\delta)^2\sum_{t=1}^T\min\cbr{\sumInAssortT \enVert{x_{t,i}}^2_{\mathbf{V}_t^{-1}} ,1}\tag{Using Cauchy-Schwarz inequality}\\
    \leq & 2\del{1+2S}\gamma_t(\delta)\sqrt{2d\log\del{1+\frac{L KT}{d\lambda_t}}T} \nonumber \\ &+ 8(1+2S)^2\kappa M \gamma_t(\delta)^2d\log\del{1+\frac{KT}{d\lambda_t}}\tag{From Lemma~\ref{lem: sums of norm square Jm} and~\ref{lem: sums of norm square Vm}}.
\end{align}
For a choice of $\lambda_t =d\log(KT)$ $\gamma_t(\delta) = \mathrm{O}\del{d^{\nicefrac{1}{2}}\log^{\nicefrac{1}{2}}\del{KT}}$.
\end{proof}

\subsection{Convex relaxation}\label{sec: convex relaxation appendix}

\begin{customlemma}{8}\label{lem: convex relaxation lemma appendix}
$E_t\del{\delta} \supseteq C_t\del{\delta}$, therefore for any $\theta \in C_t(\delta)$, we also have $\theta\,\in\,E_t(\delta)$ (see Eq~(\ref{eq: convex relaxation set prev})).
\end{customlemma}

\begin{proof}
Let $\hat{\theta}_t$ be the maximum likelihood estimate (see Eq~(\ref{eq: MLE estimate})), the second-order Taylor series expansion of the log-loss (with integral remainder term) for any $\theta \in \mathbb{R}^d$ is given by:
\begin{align}
    \mathcal{L}^{\lambda}_t(\theta) =& \mathcal{L}^{\lambda}_t(\hat{\theta}_t) + \nabla \mathcal{L}^{\lambda}_t(\hat{\theta}_t)^\top(\theta-\hat{\theta}_t)\nonumber\\ &+(\theta-\hat{\theta}_t)\del{\int^1_{\nu=0}(1-\nu)\nabla^2 \mathcal{L}^{\lambda}_t(\hat{\theta}_t+\nu(\theta-\hat{\theta}_t))\cdot d\nu}(\theta-\hat{\theta}_t) 
\end{align}
$\nabla \mathcal{L}^{\lambda}_t(\hat{\theta}_t)=0$ by definition since $\hat{\theta}_t$ is maximum likelihood estimate. Therefore :
\begin{align}
    \mathcal{L}^{\lambda}_t(\theta) &= \mathcal{L}^{\lambda}_t(\hat{\theta}_t) +(\theta-\hat{\theta}_t)^{\top}\del{\int^1_{\nu=0}(1-\nu)\nabla^2 \mathcal{L}^{\lambda}_t(\hat{\theta}_t+\nu(\theta-\hat{\theta}_t))\cdot d\nu}(\theta-\hat{\theta}_t)\notag\\
    &=\mathcal{L}^{\lambda}_t(\hat{\theta}_t) +(\theta-\hat{\theta}_t)^{\top}\del{\int^1_{\nu=0}(1-\nu)\mathbf{H}_t(\hat{\theta}_t+\nu(\theta-\hat{\theta}_t))\cdot d\nu}(\theta-\hat{\theta}_t)\tag{$\nabla^2 \mathcal{L}^{\lambda}_t(\cdot) = \mathbf{H}_t(\cdot)$}\\
    &\leq \mathcal{L}^{\lambda}_t(\hat{\theta}_t)+ \enVert{\theta - \hat{\theta}_t}^2_{\mathbf{G}_t(\theta,\hat{\theta}_t)} \tag{def. of $\mathbf{G}_t(\theta,\hat{\theta}_t)$}\\
    &\leq \mathcal{L}^{\lambda}_t(\hat{\theta}_t)+ \enVert{g_t(\theta) - g_t(\hat{\theta}_t)}^2_{\mathbf{G}^{-1}_t(\theta,\hat{\theta}_t)}. \tag{Eq~(\ref{eq: theta and G relation multinomial})}\\
\end{align}
Thus we obtain:
\begin{align}
    \mathcal{L}^{\lambda}_t(\theta) - \mathcal{L}^{\lambda}_t(\hat{\theta}_t) &\leq \enVert{g_t(\theta)-g_t(\hat{\theta_t})}^2_{\mathbf{G}_t^{-1}(\theta,\hat{\theta}_t)}\notag\\
    &\leq \del{\frac{\gamma_t^2(\delta)}{\lambda_t}+\gamma_t(\delta)}^2 = \beta^2_t(\delta),\tag{from Lemma~\ref{lem: helper lemma convex relaxation}} 
\end{align}
where the last inequality suggests that $\theta \in E_t(\delta)$ by the definition of the set $E_t(\delta)$. Therefore, $\mathbb{P}\del{\forall t \geq 1, \theta_* \in E_t(\delta)} \geq 1-\delta$.
\end{proof}

The following helper lemma, which translates the confidence set definition of Lemma~\ref{lem: bernstein confidence bound multinomial} to the norm defined by $\mathbf{G}_t^{-1}(\theta_1,\theta_2)$.
\begin{lemma}\label{lem: helper lemma convex relaxation}
Let $\delta \in (0,1]$. For all $\theta \in C_t(\delta)$ and $\hat{\theta}_t$ as the maximum likelihood estimate in Eq~(\ref{eq: MLE estimate}).
$$
    \enVert{g_t(\theta)-g_t(\hat{\theta_t})}_{\mathbf{G}_t^{-1}(\theta,\hat{\theta}_t)} \leq \frac{\gamma_t^2(\delta)}{\lambda_t}+\gamma_t(\delta).
$$
\end{lemma}
\begin{proof}
We have:
\begin{align}
    \mathbf{G}_t\del{\theta,\hat{\theta}_t} &= \sum_{s=1}^{t-1}\sumInAssortS \alpha_i\del{\designQs,\theta,\hat{\theta}_t}\xsi\xsi^\top+\lambda_t\Imatrixd\tag{def. of $\mathbf{G}_t\del{\theta,\hat{\theta}_t}$}\\
    & \geq \sum_{s=1}^{t-1}\sumInAssortS \Mmodeldot_i(\designQs^\top \theta)\del{1+|x_{s,i}^\top\theta-x_{s,i}^\top\hat{\theta}_t|}^{-1}\xsi\xsi^\top + \lambda_t\Imatrixd \tag{from Lemma~\ref{lem: application of self concordance softmax}}\\
    & \geq \sum_{s=1}^{t-1}\sumInAssortS \Mmodeldot_i(\designQs^\top \theta)\del{1+\enVert{x_{s,i}}_{\mathbf{G}^{-1}_t\del{\theta,\hat{\theta}_t}}\enVert{\theta-\hat{\theta}_t}_{\mathbf{G}_t\del{\theta,\hat{\theta}_t}}}^{-1}\xsi\xsi^\top \nonumber\\ &+ \lambda_t\Imatrixd \tag{Cauchy-Schwarz inequality}\\
    &\geq \del{1+\lambda_t^{-\nicefrac{1}{2}}\enVert{\theta-\hat{\theta}_t}_{\mathbf{G}_t\del{\theta,\hat{\theta}_t}}}^{-1}\sum_{s=1}^{t-1}\sumInAssortS \Mmodeldot_i(\designQs^\top \theta)\xsi\xsi^\top + \lambda_t\Imatrixd \tag{$\mathbf{G}_t(\theta,\hat{\theta}_t)\succeq \lambda_t\Imatrixd$}\\
    &\geq \del{1+\lambda_t^{-\nicefrac{1}{2}}\enVert{\theta-\hat{\theta}_t}_{\mathbf{G}_t\del{\theta,\hat{\theta}_t}}}^{-1}\del{\sum_{s=1}^{t-1}\sumInAssortS \Mmodeldot_i(\designQs^\top \theta)\xsi\xsi^\top + \lambda_t\Imatrixd} \notag\\
    &= \del{1+\lambda_t^{-\nicefrac{1}{2}}\enVert{\theta-\hat{\theta}_t}_{\mathbf{G}_t\del{\theta,\hat{\theta}_t}}}^{-1}\mathbf{H}_t\del{\theta} \tag{def. of $\mathbf{H}_t\del{\theta}$}\\
    &= \del{1+\lambda_t^{-\nicefrac{1}{2}}\enVert{g_t(\theta)-g_t(\hat{\theta}_t)}_{\mathbf{G}_t^{-1}\del{\theta,\hat{\theta}_t}}}^{-1}\mathbf{H}_t\del{\theta} \tag{from Eq~(\ref{eq: theta and G relation multinomial})},
\end{align}
where 
$$  \mathbf{G}_t\del{\theta,\hat{\theta}_t} \succeq \del{1+\lambda_t^{-\nicefrac{1}{2}}\enVert{g_t(\theta)-g_t(\hat{\theta}_t)}_{\mathbf{G}_t^{-1}\del{\theta,\hat{\theta}_t}}}^{-1}\mathbf{H}_t\del{\theta} $$ is analogous to local information containing counterpart of the relation in Lemma~\ref{lem: relationship between Gt and Ht}.
This gives:
\begin{align}
    &\enVert{g_t(\theta)-g_t(\hat{\theta}_t)}^2_{\mathbf{G}_t^{-1}\del{\theta,\hat{\theta}_t}}\nonumber\\ \leq & \del{1+\lambda_t^{-\nicefrac{1}{2}}\enVert{g_t(\theta)-g_t(\hat{\theta}_t)}_{\mathbf{G}_t^{-1}\del{\theta,\hat{\theta}_t}}}^{-1} \enVert{g_t(\theta)-g_t(\hat{\theta}_t)}^2_{\mathbf{H}_t^{-1}\del{\theta}}\notag\\
     \leq & \lambda_t^{-\nicefrac{1}{2}} \gamma_t^2(\delta)\enVert{g_t(\theta)-g_t(\hat{\theta}_t)}_{\mathbf{G}_t^{-1}\del{\theta,\hat{\theta}_t}}+\gamma^2_t(\delta),\tag{from Lemma~\ref{lem: bernstein confidence bound multinomial}} 
\end{align}
where the last relation is a quadratic inequality in $\enVert{g_t(\theta)-g_t(\hat{\theta}_t)}_{\mathbf{G}_t^{-1}\del{\theta,\hat{\theta}_t}}$, which on solving completes the proof of the statement in the lemma.
\end{proof}

\begin{customlemma}{9}\label{lem: convex relaxation lemma 2 appendix}
Under the event $\theta_*\,\in\,C_t(\delta)$, the following holds $\forall\,\theta\,\in\,E_t(\delta)$: 
\begin{align*}
    \enVert{\theta-\theta_*}_{\mathbf{H}_t(\theta_*)} \leq (2+2S)\gamma_t(\delta) + 2\sqrt{1+S}\beta_t(\delta).
\end{align*}
When $\lambda_t = d\log(t)$, then $\gamma_t(\delta) = \tilde{\mathrm{O}}\del{\sqrt{d\log(t)}}$, $\beta_t(\delta) = \tilde{\mathrm{O}}\del{\sqrt{d\log(t)}}$, and
\begin{equation*}
    \enVert{\theta - \theta_*}_{\mathbf{H}_t(\theta_*)} = \tilde{\mathrm{O}}\del{\sqrt{d\log(t)}}.
\end{equation*}
\end{customlemma}
\begin{proof}
Second-order Taylor expansion of the log-likelihood function with integral remainder term gives:
\begin{align}
     \mathcal{L}^{\lambda}_t(\theta) 
     = &\mathcal{L}^{\lambda}_t(\theta_*) + \nabla \mathcal{L}^{\lambda}_t(\hat{\theta}_t)^\top(\theta-\theta_*) \nonumber \\ &+(\theta-\theta_*)\del{\int^1_{\nu=0}(1-\nu)\nabla^2 \mathcal{L}^{\lambda}_t(\theta_*+\nu(\theta-\theta_*))\cdot d\nu}(\theta-\theta_*) \notag\\
     =& \mathcal{L}^{\lambda}_t(\theta_*) + \nabla \mathcal{L}^{\lambda}_t(\hat{\theta}_t)^\top(\theta-\theta_*)+\enVert{\theta-\theta_*}^2_{\mathbf{\tilde{G}}_t(\theta_*,\theta)}, \notag\\
\end{align}
where $\mathbf{\tilde{G}}_t(\theta_*,\theta) = (\theta-\theta_*)\del{\int^1_{\nu=0}(1-\nu)\mathbf{H}_t(\theta_*+\nu(\theta-\theta_*))\cdot d\nu}(\theta-\theta_*) $ . From Lemma~\ref{lem: relationship between Gt and Ht} and Lemma 8 of~\cite{abeille2020instance} is also follows that
\begin{equation*}
    \enVert{\theta-\theta_*}^2_{\mathbf{\tilde{G}}_t(\theta_*,\theta)} \geq (2+2S)^{-1}\enVert{\theta-\theta_*}^2_{\mathbf{H}_t(\theta_*)}
\end{equation*}
Therefore we have:
\begin{align}
    &\enVert{\theta-\theta_*}^2_{\mathbf{H}_t(\theta_*)} \nonumber\\ 
     \leq& (2+2S)\envert{\mathcal{L}^{\lambda}_t(\theta) - \mathcal{L}^{\lambda}_t(\theta_*)}+(2+2S)\envert{\nabla \mathcal{L}^{\lambda}_t(\hat{\theta}_t)^\top(\theta-\theta_*)} \notag \\
    \leq &2(2+2S)\beta_t^2(\delta) +(2+2S)\envert{\nabla \mathcal{L}^{\lambda}_t(\hat{\theta}_t)^\top(\theta-\theta_*)} \tag{def. of $E_t(\delta)$}\\
    \leq& 2(2+2S)\beta_t^2(\delta) +(2+2S)\enVert{\nabla \mathcal{L}^{\lambda}_t(\hat{\theta}_t)}_{\mathbf{H}_t^{-1}(\theta_*)}\enVert{(\theta-\theta_*)}_{\mathbf{H}_t(\theta_*)} \tag{Cauchy-Schwarz inequality}\\
    \leq& 2(2+2S)\beta_t^2(\delta) +(2+2S)\gamma_t(\delta)\enVert{(\theta-\theta_*)}_{\mathbf{H}_t(\theta_*)}. \notag \\
\end{align}
Solving the quadratic inequality in $\enVert{\theta-\theta_*}_{\mathbf{H}_t(\theta_*)}$, we get:
\begin{align*}
    \enVert{\theta-\theta_*}_{\mathbf{H}_t(\theta_*)} \leq (2+2S)\gamma_t(\delta) + 2\sqrt{1+S}\beta_t(\delta).
\end{align*}
When $\lambda_t = d\log(t)$, then $\gamma_t(\delta) = \tilde{\mathrm{O}}\del{\sqrt{d\log(t)}}$ and $\beta_t(\delta) = \tilde{\mathrm{O}}\del{\sqrt{d\log(t)}}$.
\end{proof}

\subsection{Technical lemmas}\label{sec: technical lemmas appendix}

\begin{remark}[Derivatives for MNL choice function]\label{lem: softmax derivative}
For the multinomial logit choice function, where the expected reward due to item $i$ of the assortment $S_t$ is modeled as:
$$
f_i(S_t,\mathbf{r}) = \frac{e^{r_i}}{1+e^{r_i}+\sum_{j\,\in\,S_t,\,j\neq i}^Ke^{r_j}} 
$$
the partial derivative with respect to the expected reward of $i_{th}$ item is given as: 
$$
\frac{\partial f_i}{\partial r_i} = f_i(S_t,\mathbf{r})\del{1-f_i(S_t,\mathbf{r})}
$$ and the double derivative as:
$$
\frac{\partial^2 f_i}{\partial r_i^2} = f_i(S_t,\mathbf{r})\del{1-f_i(S_t,\mathbf{r})}\del{1-2f_i(S_t,\mathbf{r})}.
$$
\end{remark}

\begin{lemma}[Self-Concordance like relation for MNL]\label{lem: self concordance property softmax}
For the multinomial logit choice function, where the expected reward due to item $i$ of the assortment $S_t$ is modeled as:
$$
f_i(S_t,\mathbf{r}) = \frac{e^{r_i}}{1+e^{r_i}+\sum_{j\,\in\,S_t,\,j\neq i}^Ke^{r_j}} 
$$
the following relation holds:
$$
\envert{\frac{\partial^2 f_i}{\partial r_i^2}}\leq\frac{\partial f_i}{\partial r_i}
$$
\end{lemma}
\begin{proof}
The proof directly follows from Remark~\ref{lem: softmax derivative} and the observation\\ $\envert{1-2f_i(S_t,\mathbf{r})}\leq 1$ for all items, $i$ in the assortment choice.
\end{proof}

\begin{lemma}[Generalized elliptical potential]\label{lem: elliptical potential}
Let $\cbr{\designQs}$ be a sequence in $\mathbb{R}^{d\times K}$ such that for each $s$, $\designQs$ has columns as $\{x_{s,1},x_{s,2},\cdots,x_{s,K}\}$ where $\enVert{x_{s,i}}_2\leq w,\,\in\,\mathbb{R}^d$ for all $s\geq1$ and $i\,\in\,[K]$. Also, let $\lambda_t$ be a non-negative scalar. For $t\geq1$, define $\mathbf{J}_t\coloneqq \sum_{s=1}^{t-1}\sumInAssortS \MmodeldotUnder\del{\designQs}\xsi \xsi^\top+\lambda_t\mathbf{I}_d$ where $\MmodeldotUnder\del{\designQs}$ is strictly positive for all $i,\in,[K]$. Then the following inequality holds:
$$
\sum_{t=1}^T\min\cbr{\sumInAssortT \enVert{\tilde{x}_{t,i}}^2_{\mathbf{J}^{-1}_t},1}\leq 2\log\del{\frac{\det(\mathbf{J}_{T+1})}{\lambda_t^d}}
$$
with $\tilde{x}_{t,i} = \sqrt{\MmodeldotUnder\del{\designQs}}x_{s,i}$.
\end{lemma}
\begin{proof}
By the definition of $\mathbf{J}_t$:
\begin{align*}
    \det\del{\mathbf{J}_{t+1}} &= \det\del{\mathbf{J}_t+\sumInAssortT \tilde{x}_{t,i}\tilde{x}_{t,i}^\top}\\
    &=\det\del{\mathbf{J}_t}\det\del{\mathbf{I}_d+\mathbf{J}_t^{-\nicefrac{1}{2}}\sumInAssortT \tilde{x}_{t,i}\tilde{x}_{t,i}^\top\mathbf{J}_t^{-\nicefrac{1}{2}}}\\
    &=\det\del{\mathbf{J}_t}\del{1+\sumInAssortT \enVert{\tilde{x}_{t,i}}^2_{\mathbf{J}^{-1}_t}}.
\end{align*}
Taking log from both sides and summing from $t=1$ to $T$:
\begin{align}
    \sum_{t=1}^T\log\del{1+\sumInAssortT \enVert{\tilde{x}_{t,i}}^2_{\mathbf{J}^{-1}_t}} &= \sum_{t=1}^T\log\del{\det(\mathbf{J}_{t+1})}-\log\del{\det(\mathbf{J}_{t})}\nonumber\\
    &=\sum_{t=1}^T\log\del{\frac{\mathbf{J}_{t+1}}{\mathbf{J}_{t}}}\nonumber\\
    &=\log\del{\frac{\det(\mathbf{J}_{t+1})}{\det(\lambda_t\mathbf{I}_d)}}\tag{By a telescopic sum cancellation}\\
    &=\log\del{\frac{\det(\mathbf{J}_{t+1})}{\lambda_t^d}}\label{eq: helper eq}.
\end{align}
For any $a$ such that $0\leq a \leq 1$, it follows that $a\leq 2\log(1+a)$. Therefore, we write:
\begin{align*}
    \sum_{t=1}^T\min\cbr{\sumInAssortT \enVert{\tilde{x}_{t,i}}^2_{\mathbf{J}^{-1}_t},1}&\leq 2\sum_{t=1}^T\log\del{1+\sumInAssortT \enVert{\tilde{x}_{t,i}}^2_{\mathbf{J}^{-1}_t}}\\
    &= 2\log\del{\frac{\det(\mathbf{J}_{T+1})}{\lambda_t^d}}.\tag{From Eq~(\ref{eq: helper eq})}
\end{align*}
\end{proof}

\begin{lemma}[Determinant-trace inequality, see Lemma 10 in~\cite{abbasi2011improved}]\label{lem: determinant inequality}
     Let $\{x_s\}_{s=1}^\infty$ a sequence in $\mathbb{R}^d$ such that $\enVert{x_s}_2\leq X$ for all $s\in\mathbb{N}$, and  let $\lambda_t$ be a non-negative scalar. For $t\geq 1$ define $\mathbf{V}_t \coloneqq \sum_{s=1}^{t-1} x_sx_s^\top+\lambda_t\mathbf{I}_d$. The following inequality holds:
     \begin{align*}
         \det(\mathbf{V}_{t+1}) \leq \left(\lambda_t+tX^2/d\right)^d.
     \end{align*}
\end{lemma}

\subsection{Numerical experiments}\label{app: numerical experiments}
We build on Section~\ref{sec: numerical experiments} and compare the empirical performance of our proposed algorithm~\algotwo{} with the previous state of the art in the MNL contextual bandit literature: \algocompUCB{}\citep{oh2021multinomial} and \algocompTS{}\citep{oh2019thompson} on artificial data for varying model parameters. $\theta_*\in \mathbb{R}^d$ is $d-$dimensional uniformly random variable with each coordinate in $[0,1]$ independently and uniformly distributed. The contexts follow multivariate Gaussian distribution. Algorithm~\algotwo{} only knows the value of $N,T,K,d$. Besides, algorithms~\algocompTS{}and~\algocompUCB{} also need to know the value of $\kappa$ for their implementation. Here we simulate for two additional parameter instances again averaged over $25$ Monte Carlo runs. 

\begin{figure}[h]
\centering
\includegraphics[width=.49\columnwidth]{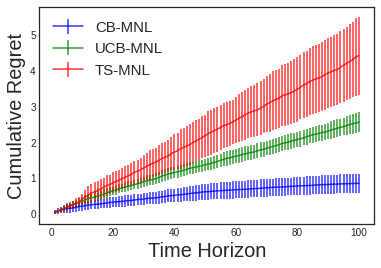}{}
\includegraphics[width=.49\columnwidth]{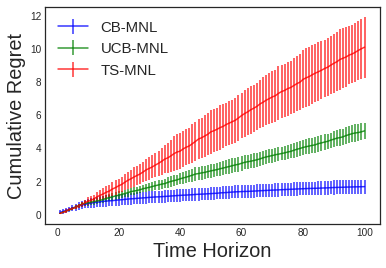}{}
\caption{Comparison of cumulative regret for two additional parameter instance ( left: $\kappa \gg \del{\sqrt{T}}$, $N=10,d=3,K=6,T=100$; right: $\kappa \gg \del{\sqrt{T}}$, $N=20,d=3,K=5,T=100$}
\label{fig: main 2}
\end{figure}

\end{document}